\documentclass{article}

\PassOptionsToPackage{numbers,sort&compress}{natbib}
\usepackage[main, preprint]{neurips_2026}

\usepackage[utf8]{inputenc}
\usepackage[T1]{fontenc}
\usepackage{microtype}
\usepackage{graphicx}
\usepackage{booktabs}
\usepackage{multirow}
\usepackage{array}
\usepackage{hyperref}
\usepackage{url}
\usepackage{amsmath}
\usepackage{mathtools}
\usepackage{amssymb}
\usepackage{amsfonts}
\usepackage{amsthm}
\usepackage{nicefrac}
\usepackage{xcolor}
\usepackage{enumitem}
\usepackage{placeins}
\usepackage[ruled,vlined,linesnumbered,noend]{algorithm2e}

\hypersetup{
  colorlinks=true,
  linkcolor=blue!55!black,
  citecolor=blue!55!black,
  urlcolor=red!60!black
}

\makeatletter
\renewcommand{\@noticestring}{%
  Preprint. Correspondence to Yao Shu \texttt{<yaoshu@hkust-gz.edu.cn>}.%
}
\makeatother

\graphicspath{{workspace/figs/}}

\newcommand{\ntxt}[1]{\text{\normalfont #1}}

\usepackage{amsmath,amsfonts,bm}

\def\1{\bm{1}}

\DeclareMathAlphabet{\mathsfit}{\encodingdefault}{\sfdefault}{m}{sl}
\SetMathAlphabet{\mathsfit}{bold}{\encodingdefault}{\sfdefault}{bx}{n}

\newcommand{\E}{\mathbb{E}}

\newcommand{\KL}{D_{\ntxt{KL}}}

\DeclareMathOperator*{\argmax}{arg\,max}
\DeclareMathOperator*{\argmin}{arg\,min}

\newcommand{\norm}[1]{\left\| #1 \right\|}

\makeatletter
\newcommand{\qty}{%
  \@ifnextchar[{\qty@square}{%
    \@ifnextchar({\qty@paren}{%
      \@ifnextchar|{\qty@bar}{\qty@parenarg}%
    }%
  }%
}
\def\qty@square[#1]{\left[#1\right]}
\def\qty@paren(#1){\left(#1\right)}
\def\qty@bar|#1|{\left|#1\right|}
\def\qty@parenarg#1{\left(#1\right)}
\makeatother

\newcommand{\bolt}{BOLT}
\newcommand{\boltfull}{Boltzmann-Targeted SFT}

\newcommand{\pref}{\pi_{\ntxt{ref}}}
\newcommand{\pistar}{\pi^*}
\newcommand{\pitheta}{\pi_\theta}
\newcommand{\objRL}{\mathcal{J}_{\ntxt{RL}}}

\newcommand{\lossPop}{\mathcal{L}}
\newcommand{\lossHat}{\hat{\mathcal{L}}}
\newcommand{\lossOracle}{\tilde{\mathcal{L}}}
\newcommand{\Zhat}{\hat{Z}}
\newcommand{\dr}{\Delta r}
\newcommand{\Ystar}{\mathcal{Y}^*}
\newcommand{\TV}{\operatorname{TV}}
\newcommand{\supp}{\operatorname{supp}}
\newcommand{\Var}{\operatorname{Var}}
\newcommand{\Btot}{B_{\ntxt{tot}}}
\newcommand{\beff}{\beta_{\ntxt{eff}}}
\newcommand{\pitop}{\pi_{\ntxt{top}}}
\newcommand{\epsopt}{\varepsilon_{\ntxt{opt}}}
\newcommand{\iid}{\overset{\ntxt{i.i.d.}}{\sim}}
\newcolumntype{L}[1]{>{\raggedright\arraybackslash}p{#1}}

\SetKwInput{KwInput}{Input}
\SetKwInput{KwOutput}{Output}
\DontPrintSemicolon

\newtheoremstyle{paperplain}
  {3pt plus 1pt minus 1pt}
  {3pt plus 1pt minus 1pt}
  {\fontfamily{LibertinusSerif-TLF}\selectfont\itshape}
  {}
  {\fontfamily{LibertinusSerif-TLF}\selectfont\bfseries}
  {.}
  {0.5em}
  {}
\theoremstyle{paperplain}
\newtheorem{theorem}{Theorem}
\newtheorem{proposition}[theorem]{Proposition}

\newtheorem{corollary}[theorem]{Corollary}
\theoremstyle{definition}

\title{Reference-Sampled Boltzmann Projection for KL-Regularized RLVR: Target-Matched Weighted SFT, Finite One-Shot Gaps, and Policy Mirror Descent}

\author{%
  Yao Shu \quad
  Chenxing Wei \quad
  Hongbin Lin \quad
  Shuang Qiu \quad
  Hui Xiong
}

\begin{document}

\maketitle

\begin{abstract}
Online reinforcement learning with verifiable rewards (RLVR) turns checkable
outcomes into a scalable training signal, but it keeps rollout generation,
verifier scoring, and reference-policy evaluations on the optimization path.
Static weighted supervised fine-tuning (SFT) on precomputed rollouts seems to
remove this bottleneck, yet a weighted likelihood is not specified by rewards
alone: its sampler and weights induce the policy being fit. This paper identifies
the reference-sampled weighted-SFT objective whose induced policy equals the
fixed-reference KL-regularized RLVR optimizer. The optimizer is the standard
Boltzmann target policy, obtained by exponentially tilting the reference policy
by verifier reward. Matching a weighted-SFT induced policy to this target forces
density-ratio weights; in the reference-sampled subclass, this reduces uniquely,
up to prompt scaling, to the prompt-normalized Boltzmann weight
\(\exp(r(x,y)/\beta)/Z(x)\). \bolt{}, a Boltzmann-Targeted SFT procedure, is the
empirical estimator of this projection. The finite one-shot analysis separates
the exact stored-support price \(\beta\log(1/\pistar(S_N\mid x))\) from
partition estimation, effective-sample-size variance, generalization,
optimization, and approximation errors. This decomposition explains why extra
SFT epochs cannot repair missing reference-policy coverage and exposes the
temperature--coverage--variance frontier. When coverage needs adaptive sampling,
refreshed Boltzmann projections become KL policy mirror descent; finite inner
solves enter as additive drift from the exact mirror step. Single-run Qwen
experiments provide projection evidence for the target-matched weight,
one-shot saturation, refreshed-sampler gains, and optimization-time savings,
within the stated single-run scope.
\end{abstract}

\section{Introduction}
\label{sec:intro}

RLVR is attractive because it turns ambiguous supervision into checkable
outcomes. In math and code, answers can be verified on GSM8K and MATH
\citep{cobbe2021gsm8k,hendrycks2021math}, and programs can be judged by unit
tests such as HumanEval \citep{chen2021evaluating}. Recent reasoning systems
use this signal at scale, with GRPO central to DeepSeekMath
\citep{shao2024deepseekmath} and DeepSeek-R1 \citep{deepseek2025r1}. The same
signal, however, is usually optimized inside an online rollout loop. PPO
alternates rollout generation and policy updates \citep{schulman2017proximal};
GRPO-style RLVR removes a separate critic but keeps on-policy rollout generation
\citep{shao2024deepseekmath}; and RLHF systems coordinate generation, scoring,
reference or old-policy log probabilities, and optimization
\citep{sheng2024hybridflow}. This loop puts the expensive part of reasoning
training on the critical path. PPO-style RLHF can require more than three times
the memory of SFT \citep{santacroce2023efficient}, long-CoT inference can
dominate wall time \citep{hu2025openrlhf}, and RLVR rollout cost grows with
model size and reasoning length \citep{huang2026pros}. Outside closed
benchmarks, rollouts may also involve live tools, mutable environments, or
long-horizon interactions whose rewards are hard to regenerate
\citep{zhang2025earlyexperience}. Verifiers simplify reward acquisition, but
online RLVR still keeps expensive rollout data inside the optimization loop.

A natural response is to decouple generation from optimization: sample verified
rollouts once, attach weights, and train with supervised infrastructure.
Reward-augmented likelihood \citep{norouzi2016raml}, advantage-weighted
regression \citep{peng2019awr}, STaR \citep{zelikman2022star}, ReST
\citep{gulcehre2023rest}, offline RL on fixed logged data
\citep{kumar2020cql,chen2021decision}, and DPO-style supervised reductions of
KL-regularized objectives \citep{rafailov2023dpo} all show that weighted or
static likelihood can be a powerful training form. They also expose the
ambiguity that matters for RLVR replacement. Weighted SFT is not one objective:
the sampler decides which rollouts can appear, the weights decide how mass is
assigned, and the product induces the policy being fit. Decoupling alone
therefore does not identify the policy replacing online RLVR. A static
weighted-SFT objective is faithful to fixed-reference KL-regularized RLVR only
when its induced policy equals the RLVR optimizer.

This paper develops \textbf{Boltzmann Projection Theory} to identify that
policy-level match. The ingredients are standard: KL-regularized reward
maximization selects a Boltzmann policy, weighted likelihood fits an induced
distribution, and relative-entropy updates underlie policy mirror descent. The
missing object is their composition for reference-sampled RLVR replacement. The
fixed-reference RLVR objective selects the Boltzmann target policy, an
exponential tilt of \(\pref\) by verifier reward. A weighted-SFT objective
selects the induced policy \(\tilde{\pi}_w\propto q w\), determined jointly by
its sampler \(q\) and weight \(w\). Equating these policies gives the
target-matching law: any sampler that covers the target needs weights
proportional to \(\pistar/q\), and the reference-sampled subclass \(q=\pref\)
therefore reduces uniquely, up to prompt scaling, to the prompt-normalized
Boltzmann density ratio \(\exp(r(x,y)/\beta)/Z(x)\)
(Sections~\ref{sec:prelim}--\ref{sec:theory-bridge}). Other weights may still
be useful heuristics, but they induce different policies and leave an
irreducible fixed-reference RLVR value gap even with infinite data and exact
optimization.

The target-matching law gives the empirical method. \bolt{}, a
Boltzmann-Targeted SFT procedure, precomputes reference rollouts, verifier
scores, Monte Carlo normalizers, and prompt-normalized Boltzmann weights, then
fits the resulting fixed weighted-likelihood objective with standard multi-epoch
SFT. This two-phase placement makes related weighted-SFT methods comparable by
their induced targets rather than by optimizer syntax: only the
reference-sampled Boltzmann density-ratio product induces the fixed-reference
RLVR target (Section~\ref{sec:method}).

With the population target fixed, the finite replacement question has a sharp
answer. A one-shot dataset \(S_N(x)\) cannot represent target mass it did not
sample, so the best support-restricted replacement pays the exact
fixed-reference gap \(\beta\log(1/\pistar(S_N\mid x))\). Conditional on
favorable support, the remaining terms are standard finite-learning costs:
partition-function estimation, effective-sample-size variance, empirical
generalization, optimization error, and parametric approximation. In binary
verifier problems, the support and variance terms reduce to a
temperature--coverage--variance frontier: making the Boltzmann target put high
mass on rare correct rollouts costs order \(1/p\) reference samples when the
reference pass probability is \(p\) (Section~\ref{sec:guarantees}). This
decomposition explains why extra SFT epochs cannot repair missing
reference-policy coverage. When coverage is weak but nonzero, the missing
operation is sampler refresh: repeating the same Boltzmann projection with the
current policy as the next reference recovers KL policy mirror descent, and
finite inner weighted-SFT solves enter as comparator drift from the exact mirror
step (Section~\ref{sec:iterative}).

These results turn practical RLVR observations into mechanism statements.
Pass@$k$ debates ask whether RLVR creates new reasoning support or concentrates
probability on already reachable solutions
\citep{yue2025rlvrboundary,wen2025implicit}; the support certificate separates
coverage from concentration. Large-scale RLVR systems invest heavily in rollout
generation, sampling rules, and sequence-level filtering
\citep{sheng2024hybridflow,hu2025openrlhf,yu2025dapo,huang2026pros}; the
projection view separates operations that change the target, operations that
change finite coverage, and operations that only change optimization cost
(Section~\ref{sec:related}). The experiments provide projection evidence rather
than leaderboard claims: they test Boltzmann density-ratio weighting against
raw-reward weighting on the same reference rollouts, the one-shot saturation
predicted by fixed support, the improvement from refreshed sampling, and the
optimization-time savings from precomputing target-defining quantities
(Section~\ref{sec:experiments}).

\section{Online RLVR and Weighted-Likelihood Induced Targets}
\label{sec:prelim}

Replacing online RLVR with static reference-rollout data first requires
specifying the policy to recover. Verifier scores alone do not determine it:
weighted SFT is also defined by the sampler that produced the scored rollouts
and by the attached weights. The replacement problem is therefore
policy-to-policy. KL-regularized RLVR maps rewards to a target policy through a
reward--KL objective. Weighted SFT maps a sampler--weight pair to an induced
target policy. Static weighted-SFT replacement is valid only when these two
policies coincide.

\subsection{RLVR Couples Rollouts and Optimization}
\label{subsec:online-rlvr}

The RLVR side of the comparison is the policy that maximizes a reward--KL
objective. Online algorithms estimate this objective with continually refreshed
samples:
PPO-style trust-region methods \citep{schulman2015trust,schulman2017proximal}
and GRPO-style RLVR \citep{shao2024deepseekmath,deepseek2025r1} sample
rollouts from the current or old policy, update the policy, and then sample
again. These rollouts estimate the objective; they are not the target policy. To
state the target separately from the training trace, fix a reference policy
$\pref(\cdot\mid x)$, let $r(x,y)\in[r_{\min},r_{\max}]$ be the verifiable
reward, and write the KL-regularized objective as
\begin{equation}
  \max_\theta\;
  \objRL(\theta)
  \triangleq
  \E_{x\sim\mathcal X,\,y\sim\pitheta(\cdot\mid x)}
  \left[r(x,y)\right]
  -\beta\,
  \E_{x\sim\mathcal X}
  \left[
  \KL\!\left(\pitheta(\cdot\mid x)\,\middle\|\,\pref(\cdot\mid x)\right)
  \right].
  \label{eq:rlvr}
\end{equation}
In \eqref{eq:rlvr}, the candidate policy $\pitheta$ supplies the completions
over which reward is averaged, while the reference policy $\pref$ only anchors
the KL penalty. Changing $\pitheta$ thus changes the reward-averaging
distribution itself. The optimizer of \eqref{eq:rlvr} is the policy-level target
for any static weighted-SFT replacement. A static objective is faithful only if
its induced policy matches this optimizer; keeping high-scoring completions is
not enough because it does not specify the conditional policy to fit.

\subsection{Weighted SFT Induces a Target Policy}
\label{subsec:weighted-sft}

Weighted SFT supplies the other map in the comparison. It starts with two
design choices: a sampler $q(\cdot\mid x)$ that provides stored rollouts and a
nonnegative weight $w(x,y)$ assigned to each rollout. Given these choices, it
optimizes the weighted negative log-likelihood
\begin{equation}
  \mathcal{L}_{q,w}(\theta)
  \triangleq
  \E_{x\sim\mathcal X,\,y\sim q(\cdot\mid x)}
  \left[-w(x,y)\log \pitheta(y\mid x)\right].
  \label{eq:weighted-sft-template}
\end{equation}
Equation~\eqref{eq:weighted-sft-template} is only a loss template. The same
form can fit different policies because $q$ decides which rollouts can appear
and $w$ decides how much mass those rollouts receive. The object comparable to
\eqref{eq:rlvr} is therefore the policy induced by the pair $(q,w)$.
Proposition~\ref{prop:weighted-sft-target} turns this ambiguity into a policy
identity: the sampler-weight product determines the conditional policy that
weighted SFT fits, which makes different weighted-SFT objectives comparable at
the policy level (proof in Appendix~\ref{app:unified}).

\begin{proposition}[Weighted-SFT Induced-Target Identity]
\label{prop:weighted-sft-target}
Let $\bar w(x)\triangleq\E_{y\sim q(\cdot\mid x)}\left[w(x,y)\right]$ and assume
$0<\bar w(x)<\infty$ almost surely. Let
$\tilde{\pi}_w(y\mid x)\triangleq q(y\mid x)w(x,y)/\bar w(x)$. Then weighted SFT
is forward-KL fitting to this induced policy:
\begin{equation}
  \mathcal{L}_{q,w}(\theta)
  =
  \underbrace{\E_x\left[
  \bar w(x)
  \KL\!\left(\tilde{\pi}_w(\cdot\mid x)\,\middle\|\,\pitheta(\cdot\mid x)\right)
  \right]
  }_{\ntxt{only term depending on }\theta}
  +
  \underbrace{\E_x\left[
  \bar w(x)
  H\!\left(\tilde{\pi}_w(\cdot\mid x)\right)
  \right]
  }_{\ntxt{constant in }\theta}.
  \label{eq:weighted-sft-induced-target}
\end{equation}
\end{proposition}
The identity is the policy-level meaning of weighted SFT. The sampler fixes
support, the weight tilts mass within that support, and the prompt multiplier
$\bar w(x)$ changes prompt importance without changing the within-prompt target.
Thus raw reward weights, advantage weights, data-ratio weights, and
density-ratio weights can share the same likelihood form while fitting different
policies. This explains why weighted-SFT ablations in practice can disagree
even when they use the same optimizer: changing the filtering rule, behavior
sampler, or scalar weight changes the induced policy before optimization
begins. The empirical distinction is therefore not training loss alone, but the
sampler-weight product \(q(y\mid x)w(x,y)\) and the support it assigns to
verified rollouts. Static reuse is faithful only when $\tilde{\pi}_w$ equals the
optimizer of \eqref{eq:rlvr}. The projection problem is now concrete: identify
that optimizer and the sampler-weight products that induce it.

\section{Boltzmann Projection for Fixed-Reference RLVR}
\label{sec:theory-bridge}

The induced-target identity makes static RLVR replacement a distribution-space
condition. A weighted-SFT objective is faithful to fixed-reference RLVR only
when its sampler-weight product induces the optimizer of the reward--KL
objective in \eqref{eq:rlvr}. That optimizer is the fixed-reference Boltzmann
target policy. Matching it requires density-ratio weights from the rollout
sampler to this target; under reference sampling, the ratio reduces to
prompt-normalized Boltzmann weights. With this equality, static reuse becomes a
mathematical replacement rather than a heuristic: the offline likelihood
objective targets the same conditional policy as the online reward--KL problem.

\subsection{Fixed-Reference Boltzmann Target}

The reward--KL objective in \eqref{eq:rlvr} selects a conditional policy, not
merely a rule for keeping high-reward completions. At each prompt, it tilts the
reference policy by the verifier reward and normalizes the tilted measure. This
Boltzmann target is standard in maximum-entropy RL,
control-as-inference, relative-entropy policy search, Gibbs variational
formulas, and DPO-style alignment
\citep{ziebart2010modeling,levine2018probabilistic,peters2010relative,donsker1975asymptotic,rafailov2023dpo}.
Here it supplies the target that any static weighted-SFT replacement has to
induce. The temperature \(\beta\) controls how far this target moves from the
reference policy: small \(\beta\) concentrates more strongly on high-verifier
rollouts, while large \(\beta\) keeps the target closer to \(\pref\).
Proposition~\ref{prop:boltzmann} states this policy explicitly (proof in
Appendix~\ref{app:distribution-proofs}).

\begin{proposition}[Fixed-Reference Boltzmann Target]
\label{prop:boltzmann}
For each prompt $x$, let the completion space be standard Borel, let
$r(x,\cdot)$ be bounded and measurable, and optimize over conditional policies
absolutely continuous with respect to $\pref(\cdot\mid x)$. The unique
prompt-wise optimizer of the inner problem in \eqref{eq:rlvr} is the Boltzmann
target policy
\begin{equation}
  \pistar(y\mid x)
  \triangleq
  \frac{\pref(y\mid x)\exp(r(x,y)/\beta)}{Z(x)},
  \qquad
  Z(x)\triangleq
  \E_{y\sim\pref(\cdot\mid x)}
  \left[\exp\!\left(r(x,y)/\beta\right)\right].
  \label{eq:boltzmann-policy}
\end{equation}
On general completion spaces, \eqref{eq:boltzmann-policy} means that the
Radon--Nikodym density of $\pistar$ with respect to $\pref$ is
$\exp(r(x,y)/\beta)/Z(x)$. Policies that charge outside the support of
$\pref(\cdot\mid x)$ have infinite reverse KL to the reference policy.
\end{proposition}

Equation~\eqref{eq:boltzmann-policy} separates the target into three components.
The reward sets the within-prompt tilt, \(Z(x)\) turns that tilt into a policy,
and absolute continuity gives the first coverage limit: fixed-reference
replacement cannot recover completions outside the support of \(\pref\). Thus
the target is value-defining, not cosmetic: Appendix~\ref{app:distribution-proofs}
proves the reverse-KL value identity that charges any population target mismatch
as fixed-reference RLVR loss. This makes \(\pistar\) the correct object
to match before discussing data reuse, normalizer estimation, or parametric
optimization; those finite errors approximate this population target rather
than replacing it. Operationally, the Boltzmann form also explains why RLVR is
not equivalent to keeping only verified successes at a fixed temperature. The KL
term deliberately leaves probability on lower-reward but reference-plausible
completions, so positive-only filtering, rejection training, or raw-reward
weighting can overconcentrate relative to the regularized RLVR objective.

\subsection{Reference-Sampled Weighted-SFT Projection}

Once \(\pistar\) is fixed, Proposition~\ref{prop:weighted-sft-target} turns
target matching into an algebraic condition on the product \(q(y\mid x)w(x,y)\).
Weighted SFT induces \(\tilde{\pi}_w\propto q w\), so matching fixed-reference
RLVR requires this product to have the same within-prompt density as
\(\pistar\). In the reference-sampled subclass \(q=\pref\), the reference policy
cancels and leaves the prompt-normalized Boltzmann density ratio. This is the
population projection in Theorem~\ref{thm:boltzmann-projection} (proof in
Appendix~\ref{app:distribution-proofs}). The cancellation is the central design
signal: the weight is not a shaped reward or an advantage surrogate, but exactly
the Radon--Nikodym derivative that converts reference rollouts into target
rollouts.
Write this reference-sampled density ratio as
\begin{equation}
  w(x,y)
  \triangleq
  \frac{\exp(r(x,y)/\beta)}{Z(x)} .
  \label{eq:boltzmann-density-ratio}
\end{equation}

\begin{theorem}[Boltzmann Projection to Reference-Sampled Weighted SFT]
\label{thm:boltzmann-projection}
Set the sampler to $q(\cdot\mid x)=\pref(\cdot\mid x)$ and use the weight in
\eqref{eq:boltzmann-density-ratio}. Then $\bar w(x)=1$, and the induced policy in
Proposition~\ref{prop:weighted-sft-target} is exactly $\pistar$. The same target
can therefore be fitted in forward KL by weighted maximum likelihood under
reference rollouts:
\begin{equation}
  \argmin_\theta\;
  \E_x\left[
  \KL\!\left(\pistar(\cdot\mid x)\,\middle\|\,\pitheta(\cdot\mid x)\right)
  \right]
  =
  \argmin_\theta\;
  \E_{x,y\sim\pref}
  \left[-w(x,y)\log \pitheta(y\mid x)\right].
\label{eq:forward-kl-bolt}
\end{equation}
\end{theorem}
This algebraic bridge is the population objective estimated by the empirical
method in Section~\ref{sec:method}. With reference rollouts, the
prompt-normalized Boltzmann weight is the density ratio \(\pistar/\pref\);
weighted maximum likelihood is therefore forward-KL fitting to the
fixed-reference RLVR target. Finite rollouts and Monte Carlo normalizers affect
estimation, but they do not change this population target. The theorem also
explains why a supervised likelihood fit can still target a reward--KL optimum:
the RLVR geometry is carried by the weights, and the likelihood fit projects
that target into the model class.
This yields a clean empirical implication: when the reference rollout set is
held fixed, changing only the weights changes the induced target. A gain from
Boltzmann density-ratio weights over raw-reward or filtered-positive weights is
then evidence for target matching rather than for a different optimizer,
dataset, or compute budget.

Reference sampling with the Boltzmann density ratio is sufficient. The
population replacement law needs the converse: for a fixed sampler \(q\), exact
replacement is possible only when the weights induce \(\pistar\). This forces
density-ratio weights up to a prompt-only scale.
Theorem~\ref{thm:boltzmann-target-characterization} states this target-matching
law; Corollary~\ref{cor:irreducible-target-gap} then states the RLVR value paid
by any mismatch (proofs in Appendix~\ref{app:unified}). This converse prevents
a common ambiguity: many weighted likelihoods have the same surface form, but
only those whose sampler-weight product equals the Boltzmann target density
recover the fixed-reference RLVR solution.

\begin{theorem}[Static Weighted-SFT Target Matching Law]
\label{thm:boltzmann-target-characterization}
Fix the sampler $q(\cdot\mid x)$ in
Proposition~\ref{prop:weighted-sft-target}, and let
$0<\bar w(x)\triangleq\E_{y\sim q(\cdot\mid x)}[w(x,y)]<\infty$ almost surely.
The induced policy matches the fixed-reference RLVR target exactly when
$\pistar(\cdot\mid x)\ll q(\cdot\mid x)$ and
\begin{equation}
  \tilde{\pi}_w=\pistar
  \quad\text{iff}\quad
  w(x,y)
  =
  \bar w(x)\frac{\pistar(y\mid x)}{q(y\mid x)}
  =
  \bar w(x)
  \frac{\pref(y\mid x)}{q(y\mid x)}
  \frac{\exp(r(x,y)/\beta)}{Z(x)},
  \label{eq:boltzmann-target-characterization}
\end{equation}
where the equivalence holds $q(\cdot\mid x)$-almost surely. The specialization
$q=\pref$ reduces this characterization to
$w(x,y)=\bar w(x)\exp(r(x,y)/\beta)/Z(x)$; the normalized convention
$\bar w(x)=1$ gives the prompt-normalized reference-sampled Boltzmann weight.
\end{theorem}

This characterization is the target-identification step. It says that weighted
SFT is not identified by a scalar reward weight alone; it is identified by the
sampler-weight product \(q(y\mid x)w(x,y)\). Inside the reference-sampled class,
this product can equal the fixed-reference Boltzmann target only through the
Boltzmann density ratio, up to a prompt-only scale. Outside that class, the same
target requires the sampler correction \(\pref/q\). The empirical method in
Section~\ref{sec:method} instantiates the reference-sampled case, where this
weight is unique inside the \(q=\pref\) subclass. Other static weighted-SFT
choices induce different policies unless they satisfy the same density-ratio
law. The practical consequence is sharper than a method taxonomy: when a
reported method changes both the sampler and the weight, its population target
has changed unless the product still implements \(\pistar\). The right
comparison is therefore class-conditional--same sampler or corrected
sampler--before claiming that one weighting rule is a faithful replacement for
fixed-reference RLVR. Corollary~\ref{cor:irreducible-target-gap} converts any
induced-target mismatch into the RLVR value loss that remains even with infinite data and exact
optimization.

\begin{corollary}[Irreducible Target-Mismatch Gap]
\label{cor:irreducible-target-gap}
For any induced policy satisfying
$\tilde{\pi}_w(\cdot\mid x)\ll\pref(\cdot\mid x)$ almost surely, the
fixed-reference RLVR gap is
\begin{equation}
  \objRL(\pistar)-\objRL(\tilde{\pi}_w)
  =
  \beta\,\E_x\left[
  \KL\!\left(\tilde{\pi}_w(\cdot\mid x)\,\middle\|\,\pistar(\cdot\mid x)\right)
  \right].
  \label{eq:irreducible-target-gap}
\end{equation}
Thus the gap is zero exactly for sampler-weight pairs that induce $\pistar$.
\end{corollary}

The gap identity makes target mismatch irreducible. Infinite data and exact
optimization fit the induced policy \(\tilde{\pi}_w\); if
\(\tilde{\pi}_w\neq\pistar\), the reverse-KL value loss in
\eqref{eq:irreducible-target-gap} remains. Static weighted SFT therefore
replaces fixed-reference RLVR exactly only when the sampler-weight product
induces \(\pistar\). The gap is a value-level attribution, not a proof artifact:
it separates errors caused by choosing the wrong population target from errors
caused by finite rollouts, Monte Carlo normalizers, model
misspecification, or incomplete optimization. Thus, when two weighted-SFT
variants use the same reference rollouts, underperformance of the mismatched
variant is naturally explained by the target-mismatch term; once the
sampler-weight product is correct, the remaining gap belongs to coverage,
normalizer, generalization, and optimization effects.

\textbf{Remark}.
The target-matching law determines only the within-prompt density. Multiplying
all weights for a prompt by \(c(x)>0\) leaves the induced policy unchanged, but
it changes how strongly that prompt contributes to the forward-KL projection.
Under reference sampling, the unit-normalized choice is the canonical
reference-sampled weight that Algorithm~\ref{alg:bolt} estimates as the
\bolt{} weight:
\begin{equation}
  w^\star_{\ntxt{BOLT}}(x,y)
  =
  \frac{\pistar(y\mid x)}{\pref(y\mid x)}
  =
  \frac{\exp(r(x,y)/\beta)}{Z(x)} .
  \label{eq:canonical-bolt-weight}
\end{equation}
The unnormalized exponentiated-reward weight uses prompt scale \(c(x)=Z(x)\).
More generally, baseline weights \(w_b(x,y)=\exp((r(x,y)-b(x))/\beta)\) preserve
the same within-prompt target for every prompt baseline \(b(x)\); the
unit-normalizer baseline is \(b(x)=\beta\log Z(x)\). Thus prompt normalization
does not change which completions are preferred within a prompt, but it removes
an extra prompt-reweighting choice. This separation is useful in finite samples:
target matching is about the conditional density, whereas prompt scaling changes
the empirical risk landscape seen by a parametric model. This explains a common
training symptom: easy prompts with many high-weight successes or hard prompts
with noisy normalizers can dominate the update even when the within-prompt
target is correct. The phenomenon appears through prompt-level total weight or
rescaling, not just through the within-prompt reward ranking.

\textbf{Remark}.
The sampler in the density ratio cannot be omitted. If rollouts come from
\(q\neq\pref\) but the weight keeps only the reference-sampled Boltzmann
numerator, the induced policy is proportional to
\(q(y\mid x)\exp(r(x,y)/\beta)\), not to
\(\pref(y\mid x)\exp(r(x,y)/\beta)\). Corollary~\ref{cor:irreducible-target-gap}
then charges this mismatch in fixed-reference RLVR value. A nonreference sampler
therefore needs the extra factor \(\pref/q\), as in
Theorem~\ref{thm:boltzmann-target-characterization}. The reference-sampled
claim is therefore deliberately scoped: the simple weight instantiated by
\bolt{} in Section~\ref{sec:method} is unique inside the
\(q=\pref\) subclass, while other samplers need their own importance correction
to target the same RLVR optimum.
This is the population reason dynamic sampling and filtering rules cannot be
interpreted only as efficiency tricks. They change \(q\); without the matching
\(\pref/q\) correction, they also change the policy being fit. The practical
question for any nonreference sampler is therefore whether the sampler is meant
to improve coverage for a new target, or to estimate the same fixed-reference
target with an explicit importance correction.

\textbf{Remark}.
A two-action example shows why raw-reward weighting is not a harmless
approximation. Fix one prompt with \(\pref(1)=\pref(0)=1/2\), \(r(1)=1\), and
\(r(0)=0\). The Boltzmann target has
\(\pistar(1)=e^{1/\beta}/(1+e^{1/\beta})\), whereas reference-sampled raw-reward
weighting with \(w=r\) induces \(\tilde\pi_w(1)=1\). Even with infinite data and
exact optimization, its fixed-reference RLVR value gap is
\(\beta\log(1/\pistar(1))>0\). The failure is target mismatch, not optimizer
error. The example also exposes the role of the KL term: the RLVR optimum does
not collapse to the best verified action unless the temperature forces that
limit, so a hard or raw-reward weighting rule can overshoot the regularized
target.

Target matching is a distribution-space statement; fitting a parametric model
adds a KL-direction boundary. Weighted likelihood fits \(\pistar\) in forward
KL, while the RLVR value identity measures reverse KL to \(\pistar\).
Realizability removes this direction mismatch. Without realizability, a fitted
policy needs a local forward-to-reverse transfer certificate such as
\(\E_x[\KL(\pitheta\|\pistar)]\le\kappa_\rho F(\theta)\), where
\(F(\theta)\triangleq
\E_x[\KL(\pistar(\cdot\mid x)\|\pitheta(\cdot\mid x))]\) is the forward-KL
target-fitting error.
This constant is not a free parameter: a sufficient condition is
\(\sup_y|\pistar(y\mid x)/\pitheta(y\mid x)-1|\le\rho<1\), which gives
\(\kappa_\rho=(1+2\rho/[3(1-\rho)^3])/(1-\rho/[3(1-\rho)^2])\). For the
Boltzmann target, \(\sup_y|\log(\pistar(y\mid x)/\pref(y\mid x))|\le\dr/\beta\);
therefore small \(\beta\), large reward range \(\dr\), or a fitted policy far
from \(\pref\) makes this value-transfer step harder. Theorem~\ref{thm:e2e}
uses the condition only to convert a fitted forward-KL projection back into RLVR
value. Appendix~\ref{app:kl-local} gives the realizable boundary, local
comparison conditions, value-transfer certificate, and projection-level
misspecification boundary
(Corollary~\ref{cor:realizability},
Propositions~\ref{prop:local-kl}--\ref{prop:local-value-transfer}, and
Corollary~\ref{cor:misspecified-local-projection}), plus a rare-action
counterexample showing why a global forward-to-reverse transfer is impossible.
Thus the population target is fixed. The finite-rollout estimator introduced in
Section~\ref{sec:method} estimates this projection rather than defining a
different RLVR target. This separation distinguishes the target itself, its
finite reference-rollout estimator, and the additional price of coverage,
normalizer estimation, generalization, optimization, and refresh. It also
explains why a low
weighted training loss is not automatically an RLVR-value certificate in a
misspecified model class. A value claim for the fitted policy needs either a
realizability or local-transfer argument, or a direct evaluation of the
reward--KL value.

\section{\bolt{} as Empirical Reference-Sampled Weighted Likelihood}
\label{sec:method}

The projection theorem fixes the population objective; an algorithm still has
to estimate it from sampled rollouts. Population target matching requires
reference rollouts and weights \(e^{r(x,y)/\beta}/Z(x)\). A finite procedure
does not know \(Z(x)\), so it estimates the prompt normalizer from the same
reference rollouts:
\begin{equation}
  \Zhat_N(x)
  \triangleq
  \frac{1}{N}\sum_{n=1}^N\exp(r(x,y_n)/\beta),
  \qquad
  \hat w(x,y_n)
  \triangleq
  \frac{\exp(r(x,y_n)/\beta)}{\Zhat_N(x)} .
  \label{eq:empirical-weights}
\end{equation}
\bolt{}, a \boltfull{} procedure, is this empirical reference-sampled
projection. Its first phase constructs a Boltzmann-weighted reference dataset:
sample rollouts from the reference policy, evaluate the automatic verifier,
estimate the prompt-wise normalizer, and store the empirical weights. Its second
phase fits the resulting fixed weighted-likelihood objective for multiple
epochs using standard supervised fine-tuning infrastructure. Algorithm~\ref{alg:bolt}
summarizes the two phases. Because reference-policy sampling, verifier scoring,
and normalization have already been precomputed, the optimization phase needs no
reference-model forward passes. The method is therefore not a new policy-gradient
update; it is the finite weighted-likelihood estimator of the population
projection in Theorem~\ref{thm:boltzmann-projection}. The only quantities that
define the RLVR target are computed before supervised optimization begins.

\begin{algorithm}[t]
\caption{\bolt{}: \boltfull{}}
\label{alg:bolt}
\KwInput{Prompts $\mathcal X_M$, rollouts per prompt $N$, initial policy
$\pi_{\theta_0}$, reward $r$, temperature $\beta$, steps $T$.}
\KwOutput{Optimized policy $\pi_\theta$.}
\tcp{Phase 1: construct the Boltzmann-weighted dataset}
Set $\pref\leftarrow\pi_{\theta_0}$\;
\For{$x\in\mathcal X_M$}{
  Sample $\{y_n\}_{n=1}^N\iid\pref(\cdot\mid x)$\;
  Evaluate $\{r(x,y_n)\}_{n=1}^N$\;
  Compute $\Zhat_N(x)$ and $\hat w(x,y_n)$ by \eqref{eq:empirical-weights}\;
  Store $\{(x,y_n,\hat w(x,y_n))\}_{n=1}^N$ in $\mathcal D$\;
}
\tcp{Phase 2: fit the fixed weighted-likelihood objective}
Initialize $\pi_\theta\leftarrow\pi_{\theta_0}$\;
\For{$t=1,\ldots,T$}{
  Sample minibatch $\mathcal B\subset\mathcal D$\;
  Compute $\ell(\theta)\triangleq-|\mathcal B|^{-1}\sum_{(x,y,w)\in\mathcal B}
  w\log\pitheta(y\mid x)$\;
  Update $\theta$ with a standard optimizer on $\ell(\theta)$\;
}
\end{algorithm}

The empirical objective has the same induced-target interpretation, but its
target is atomic on the stored rollouts. For a fixed prompt, the empirical
weights satisfy \(N^{-1}\sum_n\hat w(x,y_n)=1\), so the loss is standard
cross-entropy against
\(\hat{\pi}_{N}(y\mid x)\triangleq
N^{-1}\sum_{n=1}^N\hat w(x,y_n)\mathbf 1\{y_n=y\}\), whose support is contained
in \(\mathcal D_x=\{y_1,\ldots,y_N\}\). This single observation is the finite
counterpart of Proposition~\ref{prop:weighted-sft-target}: \bolt{} can
redistribute mass among observed rollouts, but it cannot create an unobserved
completion. The support check therefore comes before optimizer tuning. If few
verified or near-verified rollouts appear for a prompt, the stored data
represent only a small part of the Boltzmann target mass; if the empirical
target has already been fit, extra SFT epochs point to optimization error only,
whereas a one-shot plateau points to refresh or resampling.

\begin{table}[!ht]
\centering
\footnotesize
\setlength{\tabcolsep}{5pt}
\renewcommand{\arraystretch}{1.16}
\caption{Induced targets of representative LLM weighted-likelihood objectives.
For sequence-level objectives, each row applies
$\tilde{\pi}_w\propto q w$. Here $q_+$ denotes filtered generated positives,
$q_{\mathcal D}$ denotes a fixed demonstration or pre-collected data
distribution, and $\pi_0$ denotes a logged behavior policy. The DFT row records
the analogous token-level stop-gradient weighting with
$h_t=(x,y_{<t})$. The last column records whether the induced target is the
fixed-reference RLVR target $\pistar$; finite \bolt{} replaces $Z$ by
$\Zhat_N$ in Algorithm~\ref{alg:bolt}.}
\label{tab:weighted-sft-positioning}
\resizebox{0.94\linewidth}{!}{%
\begin{tabular}{@{}llllc@{}}
\toprule
Method & Sampler $q$ & Weight $w(x,y)$ & Induced target
$\tilde{\pi}_w$ & $\tilde{\pi}_w=\pistar$? \\
\midrule
STaR~\citep{zelikman2022star}
& $q_+$ & $1$ & $q_+$ & $\times$ \\
ReST~\citep{gulcehre2023rest}
& $q_{\ntxt{grow}}$ & $\mathbf{1}\{r\ge c\}$ &
$q_{\ntxt{grow}}(y\mid x)\mathbf{1}\{r\ge c\}$ & $\times$ \\
SPR~\citep{zhang2024spr}
& $\pi_k$ & $e^{(Q_k-W_k)/\beta}$ &
$\pi_k(y\mid x)e^{Q_k/\beta}$ & $\times^\dagger$ \\
VAR~\citep{du2025var}
& $q_{\mathcal D}$ & $e^{r_{\ntxt{RM}}/\beta}/\Zhat_{\mathcal B}$ &
$q_{\mathcal D}(y\mid x)e^{r_{\ntxt{RM}}/\beta}$ & $\times^\ddagger$ \\
DFT~\citep{wu2025dft}
& $q_{\mathcal D}$ & tokenwise $\ntxt{sg}\,\pitheta(y_t\mid h_t)$ &
tokenwise $q_{\mathcal D}(y_t\mid h_t)\pitheta(y_t\mid h_t)$ & $\times$ \\
Refit~\citep{mukherjee2025refit}
& $\pi_0$ & $r$ & $\pi_0(y\mid x)r(x,y)$ & $\times$ \\
\textbf{\bolt{} (ours)}
& $\pref$ & $e^{r/\beta}/Z$ & $\pistar$ &
$\checkmark$ \\
\bottomrule
\end{tabular}
\vspace{0.2em}
}
\begin{minipage}{0.97\linewidth}
\vspace{0.2em}
\scriptsize
$\dagger$ matches only in the one-shot reference-sampled case
$\pi_k=\pref$, $Q_k=r$, and exact $W_k=\beta\log Z$. \quad
$\ddagger$ VAR derives a reference-sampled learned-reward Boltzmann form, but
its implemented training uses fixed data batches; without $q_{\mathcal D}=\pref$
and verifier reward $r$, the induced target is not the fixed-reference RLVR
target studied here.
\end{minipage}
\end{table}

\FloatBarrier

Table~\ref{tab:weighted-sft-positioning} closes the population-to-algorithm
bridge. Many methods share a weighted-likelihood optimizer, but
Proposition~\ref{prop:weighted-sft-target} shows that the limiting policy is
determined by the sampler-weight product \(q w\). For \bolt{}, \(q=\pref\) and
\(w=e^{r/\beta}/Z\), so
\(q w\propto\pref(y\mid x)\exp(r(x,y)/\beta)/Z(x)\) and the induced policy is
\(\pistar\). Other products induce other policies, which
Theorem~\ref{thm:boltzmann-target-characterization} converts into fixed-reference RLVR
gaps. The finite replacement certificate therefore has a concrete target:
Algorithm~\ref{alg:bolt} estimates the only reference-sampled row in
Table~\ref{tab:weighted-sft-positioning} that removes the irreducible
target-mismatch term. The table is an induced-target comparison, not a
leaderboard: the important distinction is which conditional policy the weighted
likelihood fits in the population limit. Appendix~\ref{app:additional-consequences}
records three secondary consequences: the RLVR bias of capped density-ratio
weights, verifier-perturbation stability, and the corresponding token-level
target-matching condition (Propositions~\ref{prop:clipped-projection},
\ref{prop:noisy-verifier-stability}, and~\ref{prop:token-level-target}). The
token-level condition states that a token-weighted loss matches the same
sequence-level target only when its token coefficients equal the prefix-token
marginal density ratios induced by \(\pistar/\pref\); generic tokenwise weights
fit different local targets.

\section{Finite One-Shot Replacement Theory}
\label{sec:guarantees}

Sections~\ref{sec:theory-bridge} and~\ref{sec:method} identify the population
target and the empirical weighted likelihood that estimates it. A one-shot
replacement still pays finite-data prices because the stored reference-rollout
dataset fixes both support and estimator noise. The finite theory has three layers. The first is an
RLVR-specific support obstruction: missing Boltzmann target mass creates an
exact value loss before estimation begins. The second is an estimation
frontier: once useful rollouts appear, the same density ratio that matches the
RLVR target controls effective sample size, normalizer stability, and rollout
allocation. The third is a standard learning certificate: empirical-process and
optimization errors convert the finite weighted likelihood back into an RLVR
value bound. The order is substantive. If stored support is poor, more SFT
epochs cannot create the missing target atoms; if support is favorable, the
remaining terms are standard finite-estimation and inner-solve costs.

The support gap comes first because the Boltzmann projection can only assign
weights to generated rollouts. If $\pref$ rarely produces a near-optimal
rollout, reference-sampled weighted likelihood has no useful atom to upweight,
no matter how accurately $Z(x)$ is estimated or how well the empirical objective
is optimized. The exact distributional loss from restricting the target to
stored support is therefore the first finite gap. This is a limitation of the
sampling policy, not of the likelihood optimizer: the optimizer can only choose
among policies whose mass is expressible on the sampled support.
Theorem~\ref{thm:support-restricted-gap} states this support-restricted value
loss.

\begin{theorem}[Support-Restricted RLVR Gap]
\label{thm:support-restricted-gap}
Fix a prompt $x$, let $P\triangleq\pistar(\cdot\mid x)$, and let
$S$ be a measurable stored support with $P(S)>0$. Among all policies $Q$
supported on $S$, the smallest fixed-reference RLVR gap is
\begin{equation}
  \inf_{\supp(Q)\subseteq S}
  \left[\objRL(P)-\objRL(Q)\right]
  =
  \beta\inf_{\supp(Q)\subseteq S}\KL(Q\|P)
  =
  \beta\log\frac{1}{P(S)}.
  \label{eq:support-restricted-gap}
\end{equation}
The infimum is attained by the conditional target
$Q^*(\cdot)=P(\cdot\mid S)$. Consequently, if a useful set $A$ is missed by the
stored support, then every support-restricted one-shot method pays at least
\(\beta\log(1/(1-P(A)))\) at that prompt.
\end{theorem}

This gap is stronger than a hit-probability warning. Even an oracle optimizer
with the right weights cannot recover target mass outside the stored support.
The theorem therefore isolates a sampling limitation from an optimization
limitation.

\textbf{Remark}.
More epochs can reduce optimization error on the empirical target, but they
cannot make an absent rollout appear in \(S_N\). The relevant observable
quantities before fitting are prompt-level hit rates, pass@$k$ on the stored
support, and estimates of how much Boltzmann target mass the stored rollouts
cover. Repeated reference sampling only changes the probability that the
support event occurs. This sampling question yields the one-shot
coverage--variance frontier. Write
\(w(x,y)=d\pistar(\cdot\mid x)/d\pref(\cdot\mid x)
=\exp(r(x,y)/\beta)/Z(x)\) and
\(\mathcal C_2(x)\triangleq\E_{\pref}[w(x,y)^2]\).
The coefficient \(\mathcal C_2(x)=1+\chi^2(\pistar\|\pref)\) is the inverse
effective-sample-size factor. Thus the correct Boltzmann target may still be
expensive to estimate from the reference policy. The full Bernstein and
self-normalized forms are in Appendix~\ref{app:normalization},
Propositions~\ref{prop:coverage} and
\ref{prop:ess-oracle} / \ref{prop:snis-prompt}, plus
Corollaries~\ref{cor:rare-support-lower} and~\ref{cor:passk-support}.
For a near-optimal set write
\(A_\gamma(x)\triangleq\{y:r(x,y)\ge\max_{y'}r(x,y')-\gamma\}\) and
\(p_\gamma(x)\triangleq\pref(A_\gamma(x)\mid x)\).

\begin{theorem}[One-Shot Coverage--ESS Barrier]
\label{thm:temperature-coverage-variance}
Fix a prompt \(x\). To both observe \(A_\gamma(x)\) with probability at least
\(1-\delta\) and estimate a bounded importance-weighted prompt mean to accuracy
\(\epsilon\) after support is present, one-shot reference sampling operates at
the scale
\begin{equation}
  N
  \;\gtrsim\;
  \max\!\left\{
  \frac{\log(1/\delta)}{p_\gamma(x)},
  \frac{L^2\mathcal C_2(x)\log(1/\delta)}{\epsilon^2}
  \right\}.
\label{eq:coverage-variance-barrier}
\end{equation}
The first term comes from the exact miss probability
\(\Pr(A_\gamma(x)\cap S_N=\emptyset)=(1-p_\gamma(x))^N\); the second is the
standard Bernstein ESS term, omitting only the bounded-weight lower-order
correction. If a useful set \(A\) is missed, the support price remains
\(\beta\log(1/(1-\pistar(A\mid x)))\).
\end{theorem}

\begin{corollary}[Binary-Verifier ESS Frontier]
\label{cor:binary-ess-frontier}
For a binary verifier with \(p\triangleq\pref(r=1\mid x)\), any Boltzmann target
that assigns success probability at least \(1-\eta\) satisfies
\begin{equation}
  \mathcal C_2(x)\ge\frac{(1-\eta)^2}{p}.
  \label{eq:binary-ess-frontier}
\end{equation}
\end{corollary}

The theorem and corollary couple the two prices of one-shot reference sampling.
Useful support has to appear first; conditional on that event, density-ratio
expectations are estimated with effective sample size
\(N/\mathcal C_2(x)\). Thus a rare correct rollout cannot be both nearly
certain under \(\pistar\) and statistically cheap under \(\pref\).

\textbf{Remark}.
Temperature is not a free sharpening knob. Lower \(\beta\) concentrates the
Boltzmann target on high-reward rollouts only if such rollouts are present; it
also increases the ESS burden. In the binary case,
\(\pistar(r=1\mid x)=pa/[1+p(a-1)]\) and
\(\mathcal C_2(x)=((1-p)+pa^2)/[1+p(a-1)]^2\), with
\(p=\pref(r=1\mid x)\) and \(a=e^{1/\beta}\). Clipping is therefore a
bias--variance choice, not a neutral stabilizer: with cap \(c\), the best
capped density ratio has \(u_c(y)=\min\{\alpha w(x,y),c\}\) and pays exact RLVR
bias \(\beta\KL(Q_{u_c}\|\pistar)\) unless the cap is inactive
(Appendix~\ref{app:additional-consequences}). The rare-support and pass@$k$
specializations make the same obstruction visible in empirical summaries. In
the binary case \(r(y^+)=1\), \(r(y^-)=0\), and \(\pref(y^+\mid x)=p\), the
expected best support-restricted prompt gap is at least
\((1-p)^N\beta\log(1/(1-\pistar(y^+\mid x)))\). For any success set \(A\),
every policy supported on \(S_N\) has
\(\operatorname{pass@}k(Q;A)=1-(1-Q(A\cap S_N))^k\), so pass@$k$ is zero when
\(A\cap S_N=\emptyset\).

After support and ESS, the prompt normalizer contributes a different finite
effect. Replacing \(Z(x)\) by \(\Zhat_N(x)\) rescales the prompt loss by
\(s_N(x)=Z(x)/\Zhat_N(x)\) but leaves all within-prompt reward odds unchanged:
\(\hat w(x,y)/\hat w(x,y')=\exp((r(x,y)-r(x,y'))/\beta)\). A stable normalizer
on all-negative rollouts is therefore not a target certificate; it is a support
failure with a well-estimated low-reward sample. Once useful rollouts exist,
the finite terms identify which variable limits the replacement. The uniform
normalizer contribution scales as
\(w_{\max}L_{\log}R_\beta e^{-r_{\min}/\beta}
\sqrt{\log(M/\delta)/N}\), where
\(R_\beta=e^{r_{\max}/\beta}-e^{r_{\min}/\beta}\); hence lower temperatures
increase the same normalizer term that makes the target sharper. ESS discounts
the effective rollout count, and after coverage floors the Neyman-style
allocation scales as
\(N_i^*=\max\{N_i^{\ntxt{cov}},\lambda\sqrt{V_i}\}\) with
\(V_i=\Var_{\pref}[e^{r(x_i,y)/\beta}]\), where
\(N_i^{\ntxt{cov}}\) is the coverage floor for prompt \(x_i\) and
\(\lambda\) is the budget multiplier
(Appendix~\ref{app:normalization},
Propositions~\ref{prop:partition-concentration}/\ref{prop:uniform-normalizer},
\ref{prop:normalization}, and~\ref{prop:allocation}).

Once the estimator terms are fixed, the remaining finite errors are standard
empirical-risk and optimization components. They enter the final certificate as
learning residuals rather than as new RLVR effects. Let
\(\Theta\) be the fitted policy class and
\(\ell_\theta(x,y)\triangleq-\log\pitheta(y\mid x)\). The two population
objects entering the certificate are
\begin{equation}
\begin{aligned}
  F(\theta)
  &\triangleq
  \E_x\left[
  \KL\!\left(\pistar(\cdot\mid x)\,\middle\|\,\pitheta(\cdot\mid x)\right)
  \right],
  &
  \lossPop(\theta)
  &\triangleq
  \E_{x,y\sim\pref}\left[
  w(x,y)\ell_\theta(x,y)
  \right].
\end{aligned}
\label{eq:e2e-population-losses}
\end{equation}
Empirical-process deviation controls \(\lossPop\) versus the oracle rollout
loss, and inner optimization contributes empirical excess risk. A small
gradient norm alone is not an RLVR certificate; it diagnoses stationarity, while
Theorem~\ref{thm:e2e} uses excess risk. Appendix~\ref{app:generalization}
contains the Rademacher, PAC-Bayes, stationarity, and PL certificates
(Propositions~\ref{prop:generalization},
\ref{prop:pac-bayes-certificate}, and
\ref{prop:sgd-bound}/\ref{prop:pl-optimization}).

The final one-shot bound composes these residuals into the RLVR value metric. It
converts forward-KL learning error into the reverse-KL gap of fixed-reference
RLVR. Given
sampled prompts \(x_{1:M}\) and independent reference rollouts
\(y_{i,n}\sim\pref(\cdot\mid x_i)\) for \(i=1,\ldots,M\) and
\(n=1,\ldots,N\), define the oracle and empirical rollout losses by
\begin{equation}
\begin{aligned}
  \lossOracle_{M,N}(\theta)
  &\triangleq
  \frac1{MN}\sum_{i,n}w(x_i,y_{i,n})\ell_\theta(x_i,y_{i,n}),\\
  \lossHat_{M,N}(\theta)
  &\triangleq
  \frac1{MN}\sum_{i,n}\hat w(x_i,y_{i,n})\ell_\theta(x_i,y_{i,n}).
\end{aligned}
\label{eq:e2e-empirical-losses}
\end{equation}
Here both sums range over all sampled prompt--rollout pairs:
\(\lossOracle_{M,N}\) uses the oracle weight \(w\) with the true normalizer
\(Z(x)\), while \(\lossHat_{M,N}\) uses the empirical weight \(\hat w\) from
\eqref{eq:empirical-weights}. If \(\hat\theta\) is the output policy, set the
residuals
\begin{equation}
\begin{aligned}
  \Delta_{\ntxt{gen}}
  &\triangleq
  {\textstyle\sup_\theta}|\lossPop(\theta)-\lossOracle_{M,N}(\theta)|,\\
  \Delta_{\ntxt{norm}}
  &\triangleq
  {\textstyle\sup_\theta}|\lossOracle_{M,N}(\theta)-\lossHat_{M,N}(\theta)|,\\
  \epsopt
  &\triangleq
  \lossHat_{M,N}(\hat\theta)
  -
  {\textstyle\inf_{\theta\in\Theta}}\lossHat_{M,N}(\theta).
\end{aligned}
\label{eq:e2e-residuals}
\end{equation}
Each residual in \eqref{eq:e2e-residuals} corresponds to one transition in the
chain
\(\lossPop\to\lossOracle_{M,N}\to\lossHat_{M,N}\to\hat\theta\).

\begin{theorem}[End-to-End One-Shot Replacement Gap]
\label{thm:e2e}
Assume that the losses in \eqref{eq:e2e-empirical-losses} use independent
reference rollouts from \(\pref(\cdot\mid x_i)\) and that, for some
\(\kappa_\rho\ge1\), every candidate \(\theta\in\Theta\) to which the
certificate is applied satisfies
\(\E_x[\KL(\pitheta(\cdot\mid x)\|\pistar(\cdot\mid x))]\le\kappa_\rho F(\theta)\).
Then the fitted policy satisfies
\begin{equation}
  \objRL(\pistar)-\objRL(\pi_{\hat\theta})
  \le
  \beta\kappa_\rho
  \left({\textstyle\inf_{\theta\in\Theta}}F(\theta)+2\Delta_{\ntxt{gen}}+2\Delta_{\ntxt{norm}}+\epsopt\right).
\label{eq:e2e}
\end{equation}
\end{theorem}

The bound separates four residuals with different operational meanings. The
approximation term \(\inf_\theta F(\theta)\) is
zero under realizability and is otherwise the forward-KL projection error to the
Boltzmann target, hence it depends on the model class \(\Theta\) and the
sharpness of \(\pistar\). The multiplier \(\kappa_\rho\) comes from the local
forward-to-reverse transfer condition; it grows as the fitted
policy leaves a density-ratio neighborhood of \(\pistar\). For the finite terms,
the standard high-probability instantiation has the following variable-specific
form. With
\(\mathcal F=\{(x,y)\mapsto w(x,y)\ell_\theta(x,y):\theta\in\Theta\}\) and
\(\mathcal H=\{x\mapsto\E_{\pref}[w(x,y)\ell_\theta(x,y)]:\theta\in\Theta\}\),
Corollary~\ref{cor:e2e-components} gives
\begin{equation}
\begin{aligned}
  \Delta_{\ntxt{gen}}
  &\lesssim
  \mathfrak R_M(\mathcal H)
  +\mathfrak R_{M,N}(\mathcal F\mid x_{1:M})
  +R\sqrt{\frac{\log(1/\delta)}{M}}
  +R\sqrt{\frac{\log(1/\delta)}{MN}},\\
  \Delta_{\ntxt{norm}}
  &\lesssim
  w_{\max}L_{\log}
  R_\beta e^{-r_{\min}/\beta}
  \sqrt{\frac{\log(M/\delta)}{N}},\\
  \epsopt^{\ntxt{PL}}
  &\lesssim
  (1-\eta\mu)^T\!\left(\lossHat_{M,N}(\theta_0)-L^*\right)
  +\frac{L\eta\sigma^2}{2\mu B}.
\end{aligned}
\label{eq:e2e-term-dependencies}
\end{equation}
Here \(\mathfrak R_M(\mathcal H)\) and
\(\mathfrak R_{M,N}(\mathcal F\mid x_{1:M})\) are the prompt-level and
rollout-level empirical Rademacher complexities, \(R\) bounds
\(w(x,y)\ell_\theta(x,y)\), \(w_{\max}\) bounds the density-ratio weight,
\(L_{\log}\) bounds \(\ell_\theta\), and \(\delta\) is the failure probability. The term
\(\epsopt^{\ntxt{PL}}\) is the PL instantiation of \(\epsopt\), where
\(L^*=\inf_{\theta\in\Theta}\lossHat_{M,N}(\theta)\) and
\(\eta,\mu,L,\sigma^2,B,T\) are the step size, PL constant, smoothness,
gradient-noise variance, minibatch size, and number of optimizer steps.
Thus prompt sampling, rollout sampling, temperature-sensitive normalizer error,
model complexity, and empirical excess risk enter as different variables, not as
one aggregate concentration term. The support-conditioned certificate uses
\(\alpha_{M,N}\) for the sum of the displayed bounds for
\(2\Delta_{\ntxt{gen}}+2\Delta_{\ntxt{norm}}+\epsopt\).

\begin{corollary}[Support-Additive One-Shot Certificate]
\label{cor:coverage-conditioned}
\label{cor:support-additive-one-shot}
For prompts \(x_1,\ldots,x_M\), let \(S_N(x_i)\) be the stored support and
\(\mathcal E_\tau=\{\pistar(S_N(x_i)\mid x_i)\ge1-\tau_i\ \forall i\}\). For a
support-restricted fit, write
\(P_i^S\triangleq\pistar(\cdot\mid x_i,S_N(x_i))\) and
\(F_S(\theta)\triangleq M^{-1}\sum_i\KL(P_i^S\|\pitheta(\cdot\mid x_i))\).
On \(\mathcal E_\tau\), if the finite-error terms are bounded by
\(\alpha_{M,N}\) and the local transfer condition holds for the conditional
targets \(P_i^S\), then
\begin{equation}
  \frac{\beta}{M}\sum_{i=1}^M
  \KL\!\left(
  \pi_{\hat\theta}(\cdot\mid x_i)\,\middle\|\,
  \pistar(\cdot\mid x_i)
  \right)
  \le
  \frac{\beta}{M}\sum_{i=1}^M\log\frac{1}{1-\tau_i}
  +
  \beta\kappa_\rho
  \left(\inf_{\theta\in\Theta}F_S(\theta)+\alpha_{M,N}\right).
  \label{eq:coverage-conditioned-additive}
\end{equation}
\end{corollary}

Without the support-restricted fit, Theorem~\ref{thm:e2e} gives the corresponding
coverage-conditioned learning certificate
\(\objRL(\pistar)-\objRL(\pi_{\hat\theta})
\le\beta\kappa_\rho(\inf_{\theta\in\Theta}F(\theta)+\alpha_{M,N})\). The
support-restricted form is sharper because it keeps the exact support price
visible instead of hiding it behind a favorable-support event.

The certificate separates target identification from finite replacement. The
Boltzmann projection removes the irreducible target-mismatch term; the remaining
terms are model-class approximation, empirical sampling error, Monte Carlo
normalizer error, and empirical optimization residual. The forward/reverse-KL
assumption certifies the reverse-KL RLVR value of the fitted policy; it does not
claim global equivalence between forward and reverse projections under
misspecification. No term in \eqref{eq:e2e} is a reward-shaping correction: all
terms are prices paid after the correct target has already been identified.

\textbf{Remark}.
The support term is not hidden inside the learning terms. A wrong
sampler-weight product pays the population mismatch gap from
Corollary~\ref{cor:irreducible-target-gap}. A stored dataset that misses target
mass pays the distribution-space support price from
Theorem~\ref{thm:support-restricted-gap}. Conditional on a favorable support
event, \eqref{eq:e2e} prices the remaining learning errors.
Corollary~\ref{cor:coverage-conditioned} keeps the two prices separate by
displaying the sharp support price plus conditional fitting error. For
unrestricted neural policies, the learning bound does not absorb the support
price; the additive decomposition requires the support-restricted fit stated in
the same corollary.

\textbf{Remark}.
Equation~\eqref{eq:e2e-term-dependencies} turns the certificate into a variable
map. Rollout coverage enters through \(S_N\) and
\(p_\gamma\); estimator variance enters through the density-ratio envelope and
ESS; normalizer error worsens as \(\beta\) decreases through \(R_\beta\); model
complexity enters through Rademacher terms; and inner optimization enters
through empirical excess risk. The levers are therefore separated. Sample more or refresh when support
is missing; raise effective sample size or clip with a known target bias when
weights are unstable; allocate extra rollouts to prompts with high normalizer
variance; and reduce \(\epsopt\) only after the target and support questions are
favorable. Appendix~\ref{app:generalization} provides the proof and the
high-probability instantiations of \(\Delta_{\ntxt{gen}}\),
\(\Delta_{\ntxt{norm}}\), and \(\epsopt\).

These cases identify the point where one-shot reuse reaches its limit and
sampler refresh begins. If \(p_\gamma(x)\) is small, the direct
one-shot repair is more fixed-reference rollouts, on the order of
\(1/p_\gamma(x)\) for that prompt. Lowering \(\beta\) can sharpen weights on
sampled rollouts while increasing \(\mathcal C_2(x)\), but it cannot change the
probability that \(\pref\) produced a useful rollout. If the rollout set is
covered but unstable, the relevant quantities are ESS, prompt rescaling,
generalization, and empirical excess risk. If coverage itself is the
bottleneck, the missing operation is sampler refresh. Section~\ref{sec:iterative}
develops that refresh as KL policy mirror descent: the finite one-shot theory
says when static reuse is faithful, why a fixed sampler can still limit the
replacement, and which quantities distinguish support failure from
finite-learning error.

\section{Iterative \bolt{} as KL Policy Mirror Descent}
\label{sec:iterative}

Section~\ref{sec:guarantees} leaves one obstruction that one-shot reuse cannot
remove: a fixed rollout set cannot improve the sampler that produced it.
Iterative \bolt{} changes exactly that object. After one weighted-likelihood
round improves the policy, the improved policy becomes the next reference and
the same Boltzmann projection is applied again. In distribution space, each
round is the standard relative-entropy policy-improvement step
\citep{peters2010relative}, equivalently KL policy mirror descent
\citep{zhan2021pmd}. The PMD identity itself is standard; its role here is to
explain how refreshed reference-sampled weighted SFT continues the one-shot
projection when coverage is weak but nonzero.

The iterative theory has three moving parts. Exact refresh follows a Boltzmann
KL-PMD path, so repeated rounds accumulate reward tilts relative to the initial
policy. This path raises future rollout coverage when useful completions have
nonzero initial support. Finite inner weighted-SFT solves then enter only
through drift from the exact mirror step. Thus iteration is not more
optimization on the same stored data: it changes the sampler for the next
rollout batch.

\begin{theorem}[Refreshed Boltzmann Projection as KL Policy Mirror Descent]
\label{thm:temperature}
Suppose each outer iteration exactly minimizes its forward-KL objective, so
$\pi_{\theta_{k+1}}(y\mid x)\propto
\pi_{\theta_k}(y\mid x)\exp(r(x,y)/\beta)$. Then after $K$ iterations,
\begin{equation}
  \pi_{\theta_K}(y\mid x)
  =
  \frac{1}{C_K(x)}
  \pi_{\theta_0}(y\mid x)
  \exp\!\left(Kr(x,y)/\beta\right).
  \label{eq:exact-iterative-boltzmann-path}
\end{equation}
Equivalently, $\pi_{\theta_K}$ is the KL-regularized optimum relative to
$\pi_{\theta_0}$ at effective temperature $\beta/K$.
\end{theorem}

Theorem~\ref{thm:temperature} gives the exact path that one-shot reuse lacks.
Iteration accumulates reward tilts before future rollouts are drawn, so later
data come from a sharper policy than the initial reference. Along the
one-parameter path
\(\pi_a(y)\propto\pi_{\theta_0}(y)\exp(a r(y))\), the exact \(K\)-round
iterate is \(\pi_{K/\beta}\), and
\(d\E_{\pi_a}[r]/da=\Var_{\pi_a}(r)\ge0\). Thus a nonzero initial probability
of high-reward completions is amplified rather than merely reweighted inside a
fixed dataset. Under a strict reward gap, the exact path also has a finite-round
concentration interpretation.

\textbf{Remark}.
More epochs and more rounds have different mathematical roles. Extra epochs fit
the same stored empirical target; a refreshed round changes the sampler that
generates the next rollout batch. The effective-temperature view is therefore an
interpretation of the exact KL-PMD path; it does not define a different
objective. Several moderate KL steps can reach the same final sharpness as one
aggressive tilt while
refreshing data between steps: to target effective temperature \(\beff\), exact
iteration uses \(K/\beta\approx1/\beff\), and the nearest integer \(K\) incurs
log-density error at most \(\dr/(2\beta)\) along the Boltzmann path
(Appendix~\ref{app:iterative-proofs}). The empirical signature is the trajectory
of useful-rollout hit rates or pass rates across rounds, not only the final
checkpoint accuracy; Section~\ref{sec:experiments} reports this signature
through checkpoint trajectories.

This path also gives the data-level answer to the one-shot support barrier.
Exact KL-PMD raises the probability of a reward-separated useful set before
later rollouts are drawn, so a fixed rollout budget is spent under increasingly
better samplers instead of entirely under \(\pi_{\theta_0}\). The same
exponential-tilting principle appears in adaptive importance sampling and
cross-entropy rare-event simulation, but here the tilted sampler is the
KL-regularized RLVR mirror step \citep{rubinstein2004cross}. Refresh improves
coverage only after the useful set has positive initial support, so it extends
the one-shot regime without pretending to solve zero-support prompts.
For a useful set \(A(x)\), write \(p_k(x)\triangleq\pi_{\theta_k}(A(x)\mid x)\)
under the exact iterates and
\begin{equation}
  \underline p_k(x)
  \triangleq
  \frac{1}
  {1+\frac{1-p_0(x)}{p_0(x)}\exp(-k\gamma(x)/\beta)} .
\label{eq:refresh-lower-probability}
\end{equation}

\begin{corollary}[Budget-Matched Adaptive Refresh Coverage Gain]
\label{cor:adaptive-coverage}
Fix a prompt \(x\) and a useful set \(A(x)\). Assume
\(0<p_0(x)\triangleq\pi_{\theta_0}(A(x)\mid x)<1\) and a reward separation
\(\inf_{A}r-\sup_{A^c}r\ge\gamma(x)>0\).
Under the exact iterates in Theorem~\ref{thm:temperature}, \(p_k(x)\ge p_0(x)\)
and, for \(P_{\ntxt{refresh}}(K,N)\triangleq
1-\prod_{k=0}^{K-1}(1-p_k(x))^N\) and
\(P_{\ntxt{one}}(K,N)\triangleq1-(1-p_0(x))^{KN}\),
\begin{equation}
  P_{\ntxt{refresh}}(K,N)
  \ge
  P_{\ntxt{one}}(K,N).
  \label{eq:adaptive-coverage-hit}
\end{equation}
For \(K>1\), the inequality is strict.
\end{corollary}

\textbf{Remark}.
Corollary~\ref{cor:adaptive-coverage} formalizes the data-level advantage of
iteration over one-shot reuse, not just an optimization distinction. The
comparison is budget-matched at \(KN\) rollouts: one-shot spends all samples
under \(\pi_{\theta_0}\),
whereas refresh spends later samples under policies whose useful-set probability
has increased. Refresh does not create zero support. If \(p_0(x)=0\), neither
exact tilting nor fixed-reference one-shot sampling can hit \(A(x)\). With
nonzero initial support and a reward gap, however, each exact refresh increases
the chance that future rollouts contain the useful set. Iterative \bolt{} is
therefore useful precisely when fixed-reference coverage is weak but not absent.
Under the same reward gap, Corollary~\ref{cor:reward-gap-hit-rate} expands the
hit probability as
\begin{equation}
  P_{\ntxt{refresh}}(K,N)
  \ge
  1-\exp\!\left[
  -N\sum_{k=0}^{K-1}
  \frac{1}
  {1+\frac{1-p_0(x)}{p_0(x)}e^{-k\gamma(x)/\beta}}
  \right].
\label{eq:reward-gap-refresh-hit-main}
\end{equation}
The bound shows the mechanism: larger reward gaps \(\gamma(x)\), smaller
per-round temperature \(\beta\), and successful early refreshes increase later
sampling probabilities, whereas zero initial support \(p_0(x)=0\) leaves the
bound vacuous.

\textbf{Remark}.
The exact coverage gain survives approximate refresh as long as the actual
sampler remains close to the exact KL-PMD sampler. If the round-\(k\) sampler is
within total variation \(\eta_k(x)\) of the exact iterate, the lower bound above
loses only this drift:
\begin{equation}
  \Pr(\text{hit }A(x)\text{ at least once})
  \ge
  1-\exp\!\left(
  -N\sum_{k=0}^{K-1}\left(\underline p_k(x)-\eta_k(x)\right)_+
  \right).
\label{eq:approx-refresh-hit-main}
\end{equation}
A KL drift bound
\(\nu_k(x)\triangleq
\KL(q_{k+1}^{\ntxt{ex}}(\cdot\mid x)\|\pi_{\theta_{k+1}}(\cdot\mid x))\)
gives \(\eta_k(x)\le\sqrt{\nu_k(x)/2}\) by Pinsker. Thus, if refreshed sampling
does not improve hit rates, the issue is either insufficient exact-path support
growth or too much drift from the exact sampler. The binary verifier
specialization, strict top-set concentration rate, and effective-temperature
lattice are auxiliary interpretations; Appendix~\ref{app:iterative-proofs}
records them in Equations~\eqref{eq:binary-refresh-prob},
\eqref{eq:exponential-concentration}, and~\eqref{eq:ess-temperature-cost}.

Approximate weighted-SFT inner solves affect iteration through one quantity:
drift from the exact mirror step. The exact path remains unchanged; finite
errors enter only through the certified drift budget.
For each round, write \(q_k\triangleq\pi_{\theta_k}\) and
\(q_{k+1}^{\ntxt{ex}}(y\mid x)\propto q_k(y\mid x)\exp(r(x,y)/\beta)\) for the
exact refreshed Boltzmann target. For a comparator \(\pi^\dagger\), a certified
drift \(\varepsilon_k^\dagger\) satisfies
\begin{equation}
  \E_x\left[
  \KL\!\left(\pi^\dagger(\cdot\mid x)\,\middle\|\,\pi_{\theta_{k+1}}(\cdot\mid x)\right)
  -
  \KL\!\left(\pi^\dagger(\cdot\mid x)\,\middle\|\,q_{k+1}^{\ntxt{ex}}(\cdot\mid x)\right)
  \right]
  \le
  \varepsilon_k^\dagger .
  \label{eq:certified-comparator-drift}
\end{equation}

\begin{proposition}[Finite-Sample Drift Certificate for Refreshed \bolt{}]
\label{thm:finite-inexact-iterative}
If the certified drifts satisfy \eqref{eq:certified-comparator-drift}, then
\begin{equation}
\begin{aligned}
  &{\textstyle\sum_{k=0}^{K-1}}
  \E_x\!\left[
  \langle r(x,\cdot),\pi^\dagger(\cdot\mid x)-q_k(\cdot\mid x)\rangle
  \right]\\
  &\le
  \beta\E_x\!\left[
  \KL\!\left(\pi^\dagger(\cdot\mid x)\,\middle\|\,q_0(\cdot\mid x)\right)
  \right]
  +\frac{K\dr^2}{8\beta}
  +\beta{\textstyle\sum_{k=0}^{K-1}}\varepsilon_k^\dagger .
\end{aligned}
\label{eq:finite-inexact-iterative}
\end{equation}
\end{proposition}

\textbf{Remark}.
The drift term in \eqref{eq:certified-comparator-drift} is not an abstract
oracle quantity. A standard local inner-training certificate converts empirical
weighted-SFT excess loss into the same PMD drift budget: under local quadratic
growth \(\mu_k\), log-density Lipschitz constant \(\Lambda_k\), empirical-to-exact
log-density gap \(\zeta_k\), and inner excess loss \(\xi_k\),
\begin{equation}
  \varepsilon_k^\dagger
  \le
  \Lambda_k\sqrt{\frac{2\xi_k}{\mu_k}}+\zeta_k .
\label{eq:inner-training-drift-main}
\end{equation}
This is the iterative analogue of \(\epsopt\) in the one-shot bound: training
longer helps only insofar as it reduces this drift. Changing the rollout count,
normalizer estimator, or optimizer changes the drift budget; it does not change
the exact KL-PMD path.

For the forward-KL sufficient condition, let
\(F_k(\theta)\triangleq
\E_x[\KL(q_{k+1}^{\ntxt{ex}}(\cdot\mid x)\|\pitheta(\cdot\mid x))]\). Let
\(\Delta_{\ntxt{gen},k}\), \(\Delta_{\ntxt{norm},k}\), and \(\epsopt^{(k)}\)
denote the round-\(k\) empirical-process, normalizer, and empirical excess-loss
components.

\begin{corollary}[Forward-KL Drift Certificate]
\label{cor:forward-kl-drift}
If comparator drift is locally controlled by the forward target error,
\(\varepsilon_k^\dagger\le\kappa_kF_k(\theta_{k+1})\), then
Proposition~\ref{thm:finite-inexact-iterative} holds with
\begin{equation}
  \varepsilon_k^\dagger
  \le
  \kappa_k
  \left(
  \inf_{\theta\in\Theta}F_k(\theta)
  +2\Delta_{\ntxt{gen},k}
  +2\Delta_{\ntxt{norm},k}
  +\epsopt^{(k)}
  \right).
  \label{eq:forward-kl-drift}
\end{equation}
\end{corollary}

\textbf{Remark}.
The forward-KL statement is a sufficient way to certify drift, not the PMD
identity itself. A sufficient local transfer condition is that, for
\(P_k\triangleq q_{k+1}^{\ntxt{ex}}(\cdot\mid x)\) and
\(Q_k\triangleq\pi_{\theta_{k+1}}(\cdot\mid x)\), the comparator satisfies
\(d\pi^\dagger/dP_k\le C_k\) and the fitted policy remains in a local
density-ratio neighborhood
\(\sup_y|P_k(y)/Q_k(y)-1|\le\rho_k<1\). Then
\(\varepsilon_k^\dagger\le
C_k(1+\rho_k)(1-\rho_k)^{-1}
\sqrt{2F_k(\theta_{k+1})/(1-a_{\rho_k})}\), where
\(a_{\rho_k}\triangleq\rho_k/[3(1-\rho_k)^2]\). This condition only supplies a
route from forward-KL weighted-likelihood error to the comparator drift used in
Proposition~\ref{thm:finite-inexact-iterative}.

Corollary~\ref{cor:forward-kl-drift} separates the PMD identity from the finite
statistics used to upper-bound drift. Each refreshed round reuses the same
finite components as one-shot \bolt{}: approximation to the exact mirror
target, sampling generalization, normalizer estimation, and empirical
optimization. A stalled refresh round therefore has a structured
interpretation: the sampler may have failed to increase coverage, the empirical
target may be noisy, the normalizer may have rescaled the round badly, or the
inner solve may have drifted too far from the exact mirror step.

Appendix~\ref{app:iterative-proofs} gives the standard inexact-PMD regret and
inner-training sufficient conditions behind this drift accounting
(Proposition~\ref{thm:inexact-regret} and
Corollary~\ref{cor:inner-certificate}). Corollary~\ref{cor:adaptive-hp-iterative}
then applies the same accounting adaptively: round-\(k\) samples are fresh
conditional on the previous history, so conditional concentration plus a union
bound certifies the refreshed trajectory. In particular, if the
round-\(k\) generalization and normalizer deviations hold conditionally with
failure probabilities \(\delta_{\ntxt{gen},k}\) and
\(\delta_{\ntxt{norm},k}\), the refreshed PMD certificate holds with probability
at least
\(1-\sum_{k=0}^{K-1}(\delta_{\ntxt{gen},k}+\delta_{\ntxt{norm},k})\).
These appendix results complete the finite-data certificate without adding a
fourth main mechanism. The iterative result therefore leaves one chain: exact
refresh is KL-PMD, KL-PMD improves future coverage under nonzero support, and
finite inner solves are charged as drift from that exact path.

\section{Related Work and Positioning}
\label{sec:related}

\textbf{Online RLHF and RLVR.} Online post-training is the closest operational
setting because it optimizes a moving policy with freshly generated rollouts.
RLHF learns preference rewards \citep{christiano2017deep,ziegler2019humanpreferences},
applies them to summarization and instruction following
\citep{stiennon2020summarize,ouyang2022training}, and uses trust-region or
proximal policy updates to control KL drift
\citep{schulman2015trust,schulman2017proximal}. RLVR keeps the same
regularized policy-optimization form but replaces learned reward models with
automatic verifiers or rule-based rewards. GSM8K and MATH made mathematical
reasoning a central evaluation setting
\citep{cobbe2021gsm8k,hendrycks2021math}; DeepSeekMath and DeepSeek-R1
popularized GRPO-style RLVR \citep{shao2024deepseekmath,deepseek2025r1};
Kimi k1.5, T\"ulu 3, DAPO, and VAPO report large-scale systems, sampling, and
clipping refinements \citep{kimi2025k15,lambert2024tulu3,yu2025dapo,yue2025vapo};
and open studies such as SimpleRL-Zoo and Open-Reasoner-Zero examine RLVR
training behavior \citep{zeng2025simplerl,hu2025openreasonerzero}. These works
improve the online loop; they do not ask which fixed reference-rollout weighted
likelihood has the same population target as fixed-reference KL-regularized
RLVR. The fixed-reference scope is essential. The support analysis gives the
static side of recent pass@$k$ debates
\citep{yue2025rlvrboundary,wen2025implicit}: target-matched weighting can
concentrate probability on covered successes, but a support-restricted static
objective cannot create distribution-space support that the reference policy
never sampled.

\textbf{Static supervision, weighting, and preference objectives.} Static data
reuse and weighted likelihood are broad training forms, but their shared syntax
does not identify a shared policy target. STaR and ReST reuse generated
reasoning traces through self-training or reward filtering
\citep{zelikman2022star,gulcehre2023rest}, scaling studies train on generated
mathematical reasoning data \citep{yuan2023scaling}, and offline RL studies
fixed logged data under distribution shift through conservative value learning
or sequence modeling \citep{kumar2020cql,chen2021decision}. Reward-augmented
likelihood, advantage-weighted regression, and AWAC turn rewards or advantages
into weighted regression objectives
\citep{norouzi2016raml,peng2019awr,nair2020awac}. Preference objectives such as
DPO, KTO, ORPO, and SimPO derive supervised or reference-free losses from
preference modeling choices
\citep{rafailov2023dpo,ethayarajh2024kto,hong2024orpo,meng2024simpo}. Recent
weighted-SFT variants choose reward-, policy-, or data-ratio weights, including
VAR, SPR, Refit, and IWSFT
\citep{du2025var,zhang2024spr,mukherjee2025refit,qin2025iwsft}. \bolt{} is
closest to this static family, but its defining constraint is target matching:
under reference sampling, the induced policy has to equal the fixed-reference
Boltzmann target of KL-regularized RLVR. Thus two methods can both be weighted
SFT while optimizing different population policies because their
sampler-weight products differ.

\textbf{Regularized-RL geometry and training systems.} The geometric endpoint
used here is standard. Maximum causal entropy, control-as-inference, soft
actor-critic, relative-entropy policy search, and policy mirror descent all
connect exponentiated rewards, KL-regularized improvement, and proximal policy
updates
\citep{ziebart2010modeling,levine2018probabilistic,haarnoja2018sac,peters2010relative,zhan2021pmd}.
DPO also uses the closed-form optimum of a KL-regularized reward objective
relative to a reference policy \citep{rafailov2023dpo}. The systems motivation
is shared with parameter-efficient tuning and throughput work, including LoRA,
QLoRA, FlashAttention, HybridFlow, and DAPO
\citep{hu2022lora,dettmers2023qlora,dao2022flashattention,sheng2024hybridflow,yu2025dapo}.
The novelty boundary is therefore not the Boltzmann form or KL-PMD primitive
itself. The contribution is the reference-sampled projection bridge for RLVR:
the target-matching law for static weighted SFT, the finite one-shot gap paid by
fixed reference rollouts, and the interpretation of refreshed \bolt{} as KL
policy mirror descent rather than more optimization on the same stored data.

\section{Empirical Projection Evidence and Efficiency}
\label{sec:experiments}

The projection theory makes mechanism-level predictions, not a general
performance claim for RLVR. Large-scale RL comparisons are especially sensitive
to budgets, hyperparameters, implementation details, and random variation
\citep{patterson2024empirical}; the reported artifacts are single-run checkpoint
traces. The measurements therefore focus on the observable consequences of the
projection theory. Holding the reference-rollout source fixed while changing the
weight isolates the induced-target mechanism. Checkpoint dynamics show how far a
fixed reference-rollout dataset can be fit and how refreshed sampling changes the
best reachable checkpoint. Resource measurements report the cost consequence of
moving verifier scoring and prompt-wise normalization out of the optimization
loop. Appendix~\ref{app:experiments} gives full checkpoint curves and a
retention check.

\textbf{Protocol.} GSM8K sweeps use \(N=8\) rollouts per prompt
\citep{cobbe2021gsm8k}; the code sweep trains Qwen3.5-9B on AceCode-87K with
\(N=16\) and evaluates HumanEval \citep{chen2021evaluating}. Resource and
retention measurements use Qwen3-8B \citep{qwen3technicalreport} with LoRA rank
8, alpha 64 \citep{hu2022lora}, AdamW learning rate \(5\times10^{-6}\), weight
decay 0.1, warmup ratio 0.1, and cosine scheduling. MATH uses 12,000 training
prompts with maximum completion length 2048 \citep{hendrycks2021math}. Full
checkpoint curves and additional protocol details are in
Appendix~\ref{app:experiments}. The available logs do not include direct
prompt-level coverage, ESS, or normalizer-rescaling measurements. The finite
theory is therefore reflected here through target-weight contrasts,
fixed-objective saturation, refreshed-sampler gains, and optimization-time cost.

\subsection{Target Weighting and Optimization-Time Cost}
\label{subsec:target-weight-cost}

\begin{table}[!ht]
\centering
\small
\caption{Target-matching and resource evidence on Qwen3-8B. Accuracy is final
benchmark accuracy from the reported runs; parentheses give improvements over
the base model.
Refit uses the same reference-rollout sampler as \bolt{} but raw reward weights,
so the Refit--\bolt{} comparison isolates the target-weight change most
directly among the reported baselines. Time and peak memory are measured for the
reported runs.}
\label{tab:main-results}
\begin{tabular}{llccc}
\toprule
Benchmark & Method & Accuracy & Time & Peak memory \\
\midrule
\multirow{6}{*}{GSM8K}
& Base & 87.72 & -- & -- \\
& SFT & 88.40 (+0.68) & 2.85h & 57.59G \\
& VAR & 88.64 (+0.92) & 6.46h & 58.39G \\
& GRPO & 90.01 (+2.29) & 45.35h & 82.95G \\
& Refit & 89.23 (+1.51) & 6.96h & 57.24G \\
& \bolt{} & \textbf{90.67 (+2.95)} & 6.79h & 57.23G \\
\midrule
\multirow{6}{*}{MATH}
& Base & 55.80 & -- & -- \\
& SFT & 57.26 (+1.46) & 5.93h & 70.68G \\
& VAR & 60.17 (+4.37) & 12.82h & 71.24G \\
& GRPO & 61.37 (+5.57) & 55.68h & 83.96G \\
& Refit & 58.36 (+2.56) & 14.31h & 69.37G \\
& \bolt{} & \textbf{61.96 (+6.16)} & 13.79h & 68.96G \\
\bottomrule
\end{tabular}
\end{table}
\FloatBarrier

Table~\ref{tab:main-results} isolates the induced-target effect. Refit and
\bolt{} both train on reference-policy rollouts, but their sampler-weight
products induce different population targets: Refit uses raw reward weights,
whereas \bolt{} uses the empirical prompt-normalized Boltzmann density ratio.
The \bolt{} gains over Refit are 1.44 points on GSM8K and 3.60 points on MATH.
This is the empirical signature predicted by
Theorem~\ref{thm:boltzmann-target-characterization}: with the data source held
fixed, changing the induced target can change the fitted policy.

The resource columns show the systems consequence of the same two-phase
estimator. Against GRPO, \bolt{} reaches slightly higher final accuracy while
reducing measured training time by 85\% on GSM8K and 75\% on MATH, and peak
memory by 31\% and 18\%. These savings match Section~\ref{sec:method}:
reference-policy sampling, verifier scoring, and prompt-wise normalization have
moved out of the optimization-time loop. The comparison does not isolate every
finite-theory component. Exponentiation, prompt normalization, coverage, and
temperature are identified by the theory, but the current runs do not vary them
independently.

\subsection{One-Shot Replacement and Reference Refresh}
\label{subsec:one-shot-refresh}

\begin{table}[!ht]
\centering
\footnotesize
\setlength{\tabcolsep}{3.5pt}
\caption{One-shot and refreshed-sampler evidence. The table reports best
checkpoint accuracies from the recorded training sweeps, with optimizer step in
parentheses; full checkpoint curves are in
Appendix~\ref{app:checkpoint-curves}. The one-shot column reports the best
checkpoint reached from fixed reference rollouts, while the iterative column
reports the best checkpoint after sampler refresh.}
\label{tab:checkpoint-summary}
\begin{tabular}{lllcccc}
\toprule
Task & Model & $N$ & Base & Best non-\bolt{} & \bolt{} & Iter. \bolt{} \\
\midrule
GSM8K & Qwen3-0.6B & 8
& 48.29 & GRPO 53.17 (10k) & 54.44 (5k) & \textbf{55.46 (8k)} \\
GSM8K & Qwen3-8B & 8
& 87.72 & GRPO 90.01 (20k) & 90.67 (20k) & \textbf{91.69 (24k)} \\
GSM8K & Qwen3.5-9B & 8
& 92.03 & GRPO 93.81 (16k) & 95.12 (16k) & \textbf{96.39 (16k)} \\
HumanEval & Qwen3.5-9B & 16
& 89.39 & GRPO 90.75 (16k) & 92.39 (16k) & \textbf{94.13 (24k)} \\
\bottomrule
\end{tabular}
\end{table}
\FloatBarrier

With the target-matched objective fixed, Table~\ref{tab:checkpoint-summary}
separates one-shot replacement from sampler refresh. One-shot \bolt{} improves
over the best non-\bolt{} checkpoint in every reported sweep, indicating that
fixed reference rollouts are useful when the stored support already contains
reward-bearing completions. This is the covered regime of
Proposition~\ref{prop:coverage}: weighted likelihood can exploit stored
solutions, but it cannot remove the support barrier when the reference policy
does not sample them. Iterative \bolt{} then reaches best checkpoints 1.68--3.38
points above the best non-\bolt{} checkpoint, consistent with the KL-PMD
continuation in Theorem~\ref{thm:temperature}.
Appendix~\ref{app:checkpoint-curves} shows the corresponding learning curves:
one-shot \bolt{} often peaks and then drifts or plateaus, while refreshed
sampling pushes the best recorded checkpoint higher after the fixed-reference
objective has saturated. These learning dynamics are suggestive rather than a
budget-matched isolation of refresh as the only cause of the gain.

The empirical section therefore supports three scoped conclusions: the
prompt-normalized Boltzmann density-ratio weight changes the induced target in
the predicted direction, one-shot replacement works only in a covered
fixed-reference regime, and refreshing the sampler is the empirical operation
that realizes the KL-PMD continuation when the fixed objective saturates. The
full checkpoint grids and the secondary retention check are reported in
Appendix~\ref{app:experiments}.

\FloatBarrier

\section{Discussion and Scope}
\label{sec:discussion-scope}

The projection result identifies when a fixed reference-rollout dataset can
stand in for an online RLVR loop. The central condition is coverage. \bolt{} can
precompute verifier scores and prompt-normalized density-ratio weights, but it
cannot assign mass to completions that never appear under the reference policy.
If correct or near-correct completions have probability \(p_\gamma(x)\) close to
zero under \(\pref\), Proposition~\ref{prop:coverage} requires
\(\Omega(1/p_\gamma)\) rollouts per prompt. One-shot \bolt{} is therefore a
replacement for the covered fixed-reference regime. Sparse verifier success
under a weak reference policy requires better exploration, a stronger initial
reference policy, or refreshed sampling, not only more optimization on the same
stored rollouts.

Target matching also separates the population projection from the fitted neural
policy. In a realizable class, the prompt-normalized Boltzmann density-ratio
weight makes the weighted-SFT target coincide with the fixed-reference RLVR
target. Outside realizability, the fitted policy still pays approximation and
transfer costs: a one-sided forward-to-reverse KL bound certifies the value of a
fixed policy but does not make the forward-KL and reverse-KL projections close.
That stronger projection comparison requires the two-sided local regime in
Appendix~\ref{app:kl-local}, and the one-shot certificate in
Theorem~\ref{thm:e2e} pays the forward-KL approximation error explicitly.
Verifier quality enters at the same target level. If an imperfect verifier
changes every reward by at most \(\epsilon\), Appendix~\ref{app:additional-consequences}
shows that the induced Boltzmann target moves by sup-log-density radius at most
\(2\epsilon/\beta\), yielding true-objective target loss at most \(2\epsilon\).
BOLT therefore inherits the verifier's objective; it does not validate the
verifier itself.

The empirical results fall within this regime. The reported runs are single-run
measurements over a small set of Qwen model sizes, one main temperature, one or
two rollout counts, and one populated code benchmark. They show the
target-weight effect, the usefulness of one-shot replacement when stored support
is adequate, the gain from sampler refresh, and the optimization-time cost
reduction, but they do not identify the limiting finite term in every setting.
Broader claims require direct coverage measurements,
\(\Zhat_N(x)\) concentration checks, sweeps over \(N\) and \(\beta\),
hardware-normalized cost reporting, broader code and out-of-domain evaluations,
and budget-matched refresh ablations. Those measurements would determine where
one-shot replacement is appropriate and where KL policy mirror descent through
refresh is the better computational path.

The same scope governs deployment. Moving generation, verifier scoring, and
reference-model evaluation off the optimization path can lower the compute
barrier for adapting models to verifiable tasks. Lower cost is not itself a
safety guarantee: a cheaper verifier-optimized pipeline can also amplify harmful
objectives, brittle shortcut rewards, or reward misspecification. Practical use
therefore requires task screening, verifier stress tests, monitoring for reward
hacking, and evaluations beyond the training reward.

\bibliographystyle{plainnat}
\bibliography{workspace/reference}

\appendix
\section{Weighted-SFT Induced-Target Details}
\label{app:unified}

Weighted SFT methods can look identical at the optimizer level because they all
minimize a weighted log likelihood. The policy being fit, however, is fixed
before optimization begins. The sampler supplies support, the weights assign
mass on that support, and the product determines the induced target. Expanding
this algebra gives the main-text comparisons a policy-level interpretation:
\bolt{} is compared by the
sampler--weight product it induces, not by the shared fact that it uses a
supervised fine-tuning optimizer. The same distinction explains why raw-reward
weighting can use the same reference rollout set while fitting a different
target policy.

\begin{proof}[Proof of Proposition~\ref{prop:weighted-sft-target}]
Use
$w(x,y)=\bar w(x)\tilde{\pi}_w(y\mid x)/q(y\mid x)$ and substitute into the
weighted negative log-likelihood. The conditional loss becomes
$\bar w(x)\E_{\tilde{\pi}_w}\left[-\log\pitheta(y\mid x)\right]$, which is
$\bar w(x)[\KL(\tilde{\pi}_w\|\pitheta)+H(\tilde{\pi}_w)]$.
\end{proof}

\begin{proof}[Proof of Theorem~\ref{thm:boltzmann-target-characterization}]
If $\tilde{\pi}_w=\pistar$, then every set with zero $q$-mass also has zero
$\tilde{\pi}_w$-mass and therefore zero $\pistar$-mass, giving
$\pistar\ll q$. On the support of $q$,
\begin{equation}
  \frac{q(y\mid x)w(x,y)}{\bar w(x)}
  =
  \pistar(y\mid x),
\end{equation}
which gives the displayed density-ratio identity. Conversely, substituting that
identity into
$q(y\mid x)w(x,y)/\bar w(x)$ gives $\pistar(y\mid x)$. The
reference-sampled specialization $q=\pref$ cancels the reference-policy ratio
and gives the \bolt{} weight up to prompt scale.
\end{proof}

\begin{proof}[Proof of Corollary~\ref{cor:irreducible-target-gap}]
Apply the reverse-KL value identity \eqref{eq:reverse-kl-identity} with
$\pi=\tilde{\pi}_w$.
\end{proof}

The algebra gives a direct way to read the baselines in the experiments. For
\bolt{}, \(q\triangleq\pref\),
\(w(x,y)\triangleq\exp(r(x,y)/\beta)/Z(x)\), \(\bar w(x)=1\), and
\(\tilde{\pi}_w=\pistar\) exactly. A Refit-style baseline can hold the sampler
fixed at \(q=\pref\) but use raw reward weights \(w(x,y)\triangleq r(x,y)\);
the original Refit objective is written for a logged behavior policy
\(\pi_0\). In either case, raw reward weighting induces a reward-weighted
logged distribution rather than the Boltzmann target policy. Demonstration-based
methods fit an implicit target determined by curated samples, so their mismatch
with fixed-reference RLVR is governed by both the available curated support and
the distance between the curated target and the KL-regularized Boltzmann target
policy.

\section{Boltzmann Projection Proofs}
\label{app:distribution-proofs}

The projection chain has one role: it identifies the policy targeted by a
faithful static weighted-SFT objective. The fixed-reference reward--KL objective
selects the Boltzmann target policy; the forward projection onto that target
becomes weighted maximum likelihood under reference rollouts; and the
induced-target condition forces the Boltzmann density-ratio weight. The support
condition remains explicit throughout the proofs because it is the first place
where reference-policy coverage enters the algorithm: if \(\pref\) assigns zero
mass to a completion, neither the Boltzmann target nor a reference-sampled
estimator can recover it.

\begin{proof}[Proof of Proposition~\ref{prop:boltzmann}]
Fix $x$ and write $\mu\triangleq\pref(\cdot\mid x)$. Any feasible policy
$\pi(\cdot\mid x)\ll\mu$ has density $f=d\pi/d\mu$. Let
$Z(x)=\int \exp(r(x,y)/\beta)\,d\mu(y)$ and define
$f^*(y)=\exp(r(x,y)/\beta)/Z(x)$. Since $r$ is bounded, $Z(x)$ is finite and
positive. The prompt-wise objective can be rewritten as
\[
  \int f(y)r(x,y)\,d\mu(y)
  -
  \beta\int f(y)\log f(y)\,d\mu(y)
  =
  \beta\log Z(x)
  -
  \beta\int f(y)\log\frac{f(y)}{f^*(y)}\,d\mu(y).
\]
The final term is $\beta\KL(\pi\|\pistar)$, so the objective is maximized
uniquely at $f=f^*$ almost surely. The same display gives, after averaging over
prompts, the reverse-KL value identity
\begin{equation}
  \objRL(\pistar)-\objRL(\pi)
  =
  \beta\,\E_x\left[
  \KL\!\left(\pi(\cdot\mid x)\,\middle\|\,\pistar(\cdot\mid x)\right)
  \right].
  \label{eq:reverse-kl-identity}
\end{equation}
\end{proof}

\begin{proof}[Proof of Theorem~\ref{thm:boltzmann-projection}]
With $q=\pref$ and
$w(x,y)=\exp(r(x,y)/\beta)/Z(x)$,
$\bar w(x)=\E_{\pref}[w(x,y)]=1$. Proposition~\ref{prop:weighted-sft-target}
therefore gives
\begin{equation}
  \tilde{\pi}_w(y\mid x)
  =
  \pref(y\mid x)\frac{\exp(r(x,y)/\beta)}{Z(x)}
  =
  \pistar(y\mid x).
\end{equation}
For fixed $x$,
\begin{equation}
  \KL(\pistar\|\pitheta)
  =
  -H(\pistar)
  -\E_{y\sim\pistar(\cdot\mid x)}
  \left[\log\pitheta(y\mid x)\right].
\end{equation}
The entropy term does not depend on $\theta$. Converting the expectation under
$\pistar$ to one under $\pref$ using \eqref{eq:boltzmann-policy} gives
\begin{equation}
  \E_{y\sim\pistar}\left[\log\pitheta(y\mid x)\right]
  =
  \E_{y\sim\pref}\left[
  \frac{\exp(r(x,y)/\beta)}{Z(x)}\log\pitheta(y\mid x)
  \right].
\end{equation}
Negating gives the weighted maximum-likelihood objective.
\end{proof}

The preceding proof fixes the target policy. The remaining population question
is whether fitting that target by forward KL also certifies the reverse-KL value
used by RLVR. In a realizable model class, the direction difference disappears
because both divergences reach zero at \(\pistar\). Outside realizability, the
paper needs a local transfer condition rather than a global equivalence claim.
Corollary~\ref{cor:realizability} states the clean boundary, and
Appendix~\ref{app:kl-local} gives the local comparison used by the finite
one-shot theorem.

\begin{corollary}[Realizability]
\label{cor:realizability}
Let $\Theta$ be the model class and define
$F(\theta)\triangleq\E_x[
\KL(\pistar(\cdot\mid x)\,\|\,\pitheta(\cdot\mid x))]$.
If $\pistar(\cdot\mid x)\in\{\pitheta(\cdot\mid x):\theta\in\Theta\}$ for
almost every $x$, then every global minimizer of $F$ satisfies
$\pitheta=\pistar$ almost surely and therefore maximizes $\objRL$.
\end{corollary}

\begin{proof}[Proof of Corollary~\ref{cor:realizability}]
The forward KL is nonnegative and equals zero if and only if
$\pitheta(\cdot\mid x)=\pistar(\cdot\mid x)$ almost surely. If $\pistar$ is in
the model class, the global minimum of $F$ is zero. The reverse-KL identity in
\eqref{eq:reverse-kl-identity} then shows that the same policy maximizes the
original KL-regularized RLVR objective.
\end{proof}

\section{Local Forward/Reverse-KL Comparability}
\label{app:kl-local}

Outside realizability, target matching and value certification become separate
questions. Weighted likelihood fits the Boltzmann target by forward KL, whereas
the fixed-reference RLVR value gap is reverse KL to the same target. Nearby
distributions have the same second-order information-geometric geometry
\citep{amari2016information}, but that local fact is useful only when the fitted
policy remains in a controlled density-ratio neighborhood of \(\pistar\). The
boundary is exact: the local comparison supports value transfer for a fitted
policy, while a stronger two-sided condition is needed before comparing the
forward and reverse population projections under misspecification.

\begin{proposition}[Local Forward/Reverse KL Comparison]
\label{prop:local-kl}
For any prompt, let $P\triangleq\pistar(\cdot\mid x)$ and
$Q\triangleq\pitheta(\cdot\mid x)$. Define
$u(y)\triangleq P(y)/Q(y)-1$ and assume
$\rho\triangleq\sup_y|u(y)|<1$. Then
\begin{equation}
\begin{aligned}
  \KL(P\|Q)
  &=
  \frac12\chi^2(P\|Q)+R_f,
  &
  \KL(Q\|P)
  &=
  \frac12\chi^2(P\|Q)+R_r,
  \\
  |R_f|
  &\le
  \frac{\rho}{3(1-\rho)^2}\frac12\chi^2(P\|Q),
  &
  |R_r|
  &\le
  \frac{2\rho}{3(1-\rho)^3}\frac12\chi^2(P\|Q).
\end{aligned}
\label{eq:local-kl-comparison}
\end{equation}
Here
\(\chi^2(P\|Q)\triangleq\sum_y Q(y)(P(y)/Q(y)-1)^2\).
\end{proposition}

\begin{proof}[Proof of Proposition~\ref{prop:local-kl}]
Write $P(y)=Q(y)(1+u(y))$ with $\sum_y Q(y)u(y)=0$. For the forward KL,
\begin{equation}
  \KL(P\|Q)=\sum_y Q(y)(1+u(y))\log(1+u(y)).
\end{equation}
The Taylor expansion
$(1+u)\log(1+u)=u+\frac12u^2+h_f(u)$ has
$h_f'''(u)=-(1+u)^{-2}$, so
$|h_f(u)|\le |u|^3/(6(1-\rho)^2)$. This yields the stated $R_f$ bound.
For reverse KL,
$-\log(1+u)=-u+\frac12u^2+h_r(u)$ with
$h_r'''(u)=-2(1+u)^{-3}$, giving
$|h_r(u)|\le |u|^3/(3(1-\rho)^3)$ and the stated $R_r$ bound.
\end{proof}

Proposition~\ref{prop:local-kl} is deliberately local. It shows that, when the
fitted policy is already close to \(\pistar\) in density ratio, forward and
reverse KL have comparable quadratic cores and controlled remainders. The
one-sided value certificate applies this comparison to one fitted policy. It does not
claim that minimizing forward KL over a misspecified class always gives the
same projection as minimizing reverse KL; that stronger statement needs both
candidate projections to remain in the same local regime.

\begin{proposition}[One-Sided Value Transfer for a Fitted Policy]
\label{prop:local-value-transfer}
Let \(R(\theta)\triangleq
\E_x[\KL(\pitheta(\cdot\mid x)\|\pistar(\cdot\mid x))]\) and
\(F(\theta)\triangleq
\E_x[\KL(\pistar(\cdot\mid x)\|\pitheta(\cdot\mid x))]\). If a fitted policy
\(\theta\) satisfies
\(\sup_y|\pistar(y\mid x)/\pitheta(y\mid x)-1|\le\rho<1\) almost surely, and
$a_\rho\triangleq\rho/(3(1-\rho)^2)<1$, then
\begin{equation}
  R(\theta)
  \le
  \kappa_\rho F(\theta),
  \qquad
  \kappa_\rho\triangleq
  \frac{1+2\rho/(3(1-\rho)^3)}{1-a_\rho}.
\label{eq:local-value-transfer}
\end{equation}
Consequently
$\objRL(\pistar)-\objRL(\pitheta)\le\beta\kappa_\rho F(\theta)$. A sufficient
log-ratio trust-region condition is
\(\sup_y|\log(\pitheta(y\mid x)/\pistar(y\mid x))|\le B<\log 2\), with
\(\rho=e^B-1\).
\end{proposition}

\begin{proof}[Proof of Proposition~\ref{prop:local-value-transfer}]
For each prompt, Proposition~\ref{prop:local-kl} gives
\[
  \KL(P\|Q)
  \ge
  (1-a_\rho)\frac12\chi^2(P\|Q),
  \qquad
  \KL(Q\|P)
  \le
  (1+b_\rho)\frac12\chi^2(P\|Q),
\]
where
$a_\rho\triangleq\rho/(3(1-\rho)^2)$ and
$b_\rho\triangleq2\rho/(3(1-\rho)^3)$. Averaging over prompts yields
$R(\theta)\le ((1+b_\rho)/(1-a_\rho))F(\theta)$. The value statement follows
from \eqref{eq:reverse-kl-identity}.
If
$|\log(\pitheta(y\mid x)/\pistar(y\mid x))|\le B$, then
$|u(y)|=|\pistar(y\mid x)/\pitheta(y\mid x)-1|\le e^B-1$. Thus
$B<\log2$ implies $\rho=e^B-1<1$.
\end{proof}

For the Boltzmann target, the local condition has a concrete temperature
reading. Since
$\log(\pistar(y\mid x)/\pref(y\mid x))=r(x,y)/\beta-\log Z(x)$ and
$\log Z(x)\in[r_{\min}/\beta,r_{\max}/\beta]$, the target is a bounded
multiplicative tilt of the reference policy:
$\sup_y|\log(\pistar(y\mid x)/\pref(y\mid x))|\le\dr/\beta$, where
$\dr\triangleq r_{\max}-r_{\min}$. If a fitted policy also stays within a
sup-log-density radius $\tau$ of the reference policy, then
$\sup_y|\log(\pistar(y\mid x)/\pitheta(y\mid x))|\le\dr/\beta+\tau$. The
assumption $\rho<1$ is therefore implied by the explicit local regime
$\dr/\beta+\tau<\log 2$. This is intentionally a local statement: small
temperatures or large reward ranges can leave this regime before optimization
begins, which is why the main one-shot theorem pays an explicit forward
approximation term rather than claiming global forward/reverse projection
equivalence.

\textbf{Remark}.
A rare-action example shows why the local condition is necessary. On two
actions, let
$P_\epsilon=(\epsilon,1-\epsilon)$ and
$Q_\epsilon=(\delta_\epsilon,1-\delta_\epsilon)$ with
$\delta_\epsilon\downarrow0$ and
$\epsilon=\exp(-1/\delta_\epsilon^2)$. Then
$\KL(P_\epsilon\|Q_\epsilon)\to0$, but
$\KL(Q_\epsilon\|P_\epsilon)\to\infty$. Thus a small forward KL can still place
too much mass on a target-rare action, which is exactly the failure mode that
reverse-KL RLVR value penalizes.

The one-sided certificate is enough for Theorem~\ref{thm:e2e}, which evaluates
the learned policy after fitting. A projection-level comparison asks for more:
if the forward and reverse KL minimizers over a misspecified class are both
local to the Boltzmann target, then their values are comparable.
Corollary~\ref{cor:misspecified-local-projection} keeps this stronger claim
separate from the main text so that the misspecified case is not overstated.

\begin{corollary}[Local Projection Transfer under Misspecification]
\label{cor:misspecified-local-projection}
Choose population projections
$\theta_{\ntxt{F}}\in\argmin_{\theta\in\Theta}F(\theta)$ and
$\theta_{\ntxt{R}}\in\argmin_{\theta\in\Theta}R(\theta)$.
Suppose that both projections satisfy, almost surely with a common $\rho<1$,
\(\sup_y|\pistar(y\mid x)/\pitheta(y\mid x)-1|\le \rho\) for
\(\theta\in\{\theta_{\ntxt{F}},\theta_{\ntxt{R}}\}\).
Set
$a_\rho\triangleq\rho/(3(1-\rho)^2)$ and
$b_\rho\triangleq2\rho/(3(1-\rho)^3)$, assume
$a_\rho<1$ and $b_\rho<1$, and define
\(\Gamma_\rho\triangleq(1+b_\rho)(1+a_\rho)/[(1-a_\rho)(1-b_\rho)]\).
Then
\begin{equation}
  R(\theta_{\ntxt{F}})
  \le
  \Gamma_\rho R(\theta_{\ntxt{R}}),
  \qquad
  \objRL(\pi_{\theta_{\ntxt{R}}})-\objRL(\pi_{\theta_{\ntxt{F}}})
  \le
  \beta(\Gamma_\rho-1)R(\theta_{\ntxt{R}}).
\label{eq:misspecified-local-projection}
\end{equation}
\end{corollary}

\begin{proof}[Proof of Corollary~\ref{cor:misspecified-local-projection}]
For a fixed
$\theta$, write
\[
  C(\theta)
  \triangleq
  \E_x\left[
  \chi^2\!\left(\pistar(\cdot\mid x)\,\middle\|\,\pitheta(\cdot\mid x)\right)/2
  \right].
\]
Proposition~\ref{prop:local-kl} gives, after averaging over prompts,
\[
  (1-a_\rho)C(\theta)\le F(\theta)\le(1+a_\rho)C(\theta)
\]
and
\[
  (1-b_\rho)C(\theta)\le R(\theta)\le(1+b_\rho)C(\theta),
  \qquad
  b_\rho\triangleq\frac{2\rho}{3(1-\rho)^3}.
\]
The lower bounds use $a_\rho<1$ and $b_\rho<1$. Applying the upper transfer to
$\theta_{\ntxt{F}}$, the optimality of $\theta_{\ntxt{F}}$ for $F$, and then the
lower transfer for $\theta_{\ntxt{R}}$ gives
\begin{align*}
  R(\theta_{\ntxt{F}})
  &\le
  \frac{1+b_\rho}{1-a_\rho}F(\theta_{\ntxt{F}})
  \le
  \frac{1+b_\rho}{1-a_\rho}F(\theta_{\ntxt{R}}) \\
  &\le
  \frac{(1+b_\rho)(1+a_\rho)}
       {(1-a_\rho)(1-b_\rho)}
  R(\theta_{\ntxt{R}})
  =
  \Gamma_\rho R(\theta_{\ntxt{R}}).
\end{align*}
Finally, by \eqref{eq:reverse-kl-identity},
\[
  \objRL(\pi_{\theta_{\ntxt{R}}})-\objRL(\pi_{\theta_{\ntxt{F}}})
  =
  \beta\left(R(\theta_{\ntxt{F}})-R(\theta_{\ntxt{R}})\right)
  \le
  \beta(\Gamma_\rho-1)R(\theta_{\ntxt{R}}),
\]
which proves the stated value-excess bound.
\end{proof}

\section{Reference Coverage and Normalizer Estimation}
\label{app:normalization}

Coverage and normalizer estimation both arise during rollout generation, but
they answer different questions. Coverage asks whether useful completions appear
in the stored set at all. Normalizer estimation asks whether the weights on
already observed completions are rescaled accurately. The distinction matters
because missing support is not a concentration error: a sharply estimated
normalizer on all-negative rollouts still leaves the target mass outside the
dataset. The estimator components therefore separate support, self-normalized
rescaling, and prompt allocation before they are combined in the one-shot
certificate.

The first estimator component treats the oracle density ratio
\(w=d\pistar/d\pref\) as known. Its variance is governed by the second moment of
that ratio, so the same weight that matches the RLVR target can reduce the
effective sample size when the target is far from the reference policy.

\begin{proposition}[ESS-Controlled Estimation Component]
\label{prop:ess-oracle}
Fix a prompt \(x\). Let \(y_1,\ldots,y_N\iid\pref(\cdot\mid x)\),
\(w(x,y)=d\pistar(\cdot\mid x)/d\pref(\cdot\mid x)\),
\(\mathcal C_2(x)\triangleq\E_{\pref}[w(x,y)^2]\), and
\(w_{\max}(x)\triangleq\sup_y w(x,y)\). For any loss
\(0\le\ell(y)\le L\), define
\(\mu\triangleq\E_{\pistar}[\ell(y)]\) and
\(\hat\mu_N\triangleq N^{-1}\sum_{n=1}^N w(x,y_n)\ell(y_n)\).
Then, with probability at least \(1-\delta\),
\begin{equation}
  |\hat\mu_N-\mu|
  \le
  L\sqrt{\frac{2\mathcal C_2(x)\log(2/\delta)}{N}}
  +
  \frac{2Lw_{\max}(x)\log(2/\delta)}{3N}.
  \label{eq:ess-oracle}
\end{equation}
\end{proposition}

\begin{corollary}[Self-Normalized Prompt Perturbation]
\label{prop:snis-prompt}
Under Proposition~\ref{prop:ess-oracle}, let
\(\hat\mu_N^{\ntxt{sn}}\triangleq
\sum_{n=1}^N w(x,y_n)\ell(y_n)/\sum_{n=1}^N w(x,y_n)\). If
\(\hat Z_w\triangleq N^{-1}\sum_{n=1}^Nw(x,y_n)\) satisfies
\(|\hat Z_w-1|\le\epsilon<1\), then
\begin{equation}
  |\hat\mu_N^{\ntxt{sn}}-\mu|
  \le
  \frac{|\hat\mu_N-\mu|+L\epsilon}{1-\epsilon}.
  \label{eq:snis-prompt}
\end{equation}
\end{corollary}

Self-normalization is the empirical form used by \bolt{} when \(Z(x)\) is
unknown. Corollary~\ref{prop:snis-prompt} shows that this operation mainly
introduces prompt-level rescaling: once the empirical total weight is close to
one, the normalized prompt mean is close to the oracle importance-weighted
prompt mean. Proposition~\ref{prop:partition-concentration} and
Corollary~\ref{prop:uniform-normalizer} specialize this rescaling to the Monte
Carlo partition estimate \(\Zhat_N(x)\).

\begin{proposition}[Standard Prompt Normalizer Concentration]
\label{prop:partition-concentration}
Assume \(r(x,y)\in[r_{\min},r_{\max}]\) and define
\(R_\beta\triangleq\exp(r_{\max}/\beta)-\exp(r_{\min}/\beta)\).
For a fixed prompt \(x\), if
\(\Zhat_N(x)=N^{-1}\sum_{n=1}^N\exp(r(x,y_n)/\beta)\), then
\begin{equation}
  \Pr\!\left(|\Zhat_N(x)-Z(x)|\ge t\right)
  \le
  2\exp\!\left(-\frac{2Nt^2}{R_\beta^2}\right).
  \label{eq:partition-concentration}
\end{equation}
\end{proposition}

\begin{corollary}[Uniform Prompt-Normalizer Loss Perturbation]
\label{prop:uniform-normalizer}
Under Proposition~\ref{prop:partition-concentration}, for prompts
\(x_1,\ldots,x_M\), set
\(s_i\triangleq Z(x_i)/\Zhat_N(x_i)\). If
\(0\le\ell_\theta(x,y)\le L_{\log}\) and
\(w(x,y)\le w_{\max}\), then with probability at least \(1-\delta\),
\begin{equation}
  \sup_\theta
  |\lossOracle_{M,N}(\theta)-\lossHat_{M,N}(\theta)|
  \le
  w_{\max}L_{\log}
  \frac{R_\beta e^{-r_{\min}/\beta}}{\sqrt{2N}}
  \sqrt{\log\frac{2M}{\delta}} .
  \label{eq:uniform-normalizer}
\end{equation}
\end{corollary}

These normalizer bounds are not support guarantees. They control
the loss perturbation caused by replacing \(Z(x)\) with \(\Zhat_N(x)\) on the
rollouts that were actually sampled. The support barrier is a different event:
whether the stored set contains the reward-bearing region that the Boltzmann
target would upweight.

\begin{proposition}[Reference Coverage and One-Shot Support Barrier]
\label{prop:coverage}
Fix a prompt \(x\) and define
\(A_\gamma(x)\triangleq\{y:r(x,y)\ge \max_{y'}r(x,y')-\gamma\}\) and
\(p_\gamma(x)\triangleq\pref(A_\gamma(x)\mid x)\).
For \(N\) independent reference rollouts, the probability of missing
\(A_\gamma(x)\) is \((1-p_\gamma(x))^N\). Thus hitting
\(A_\gamma(x)\) with probability at least \(1-\delta\) requires
\begin{equation}
  N
  \ge
  \frac{\log(1/\delta)}{-\log(1-p_\gamma(x))}
  \ge
  \frac{1-p_\gamma(x)}{p_\gamma(x)}\log\frac1\delta .
\label{eq:coverage-sample-lower}
\end{equation}
For \(M\) prompts with
\(p_{\gamma,\min}\triangleq\min_i p_\gamma(x_i)>0\), the condition
\(N\ge\log(M/\delta)/p_{\gamma,\min}\) suffices to hit every
\(A_\gamma(x_i)\) with probability at least
\(1-\delta\). Moreover, any empirical target formed by nonnegative weights on
the stored rollout set \(\mathcal D_x\) is supported on \(\mathcal D_x\); if
\(\mathcal D_x\cap A_\gamma(x)=\emptyset\), then it assigns zero mass to
\(A_\gamma(x)\).
\end{proposition}

\begin{corollary}[Rare-Support Lower Bound]
\label{cor:rare-support-lower}
For a binary prompt with completions \(y^+\) and \(y^-\), rewards
\(r(y^+)=1\), \(r(y^-)=0\), and \(\pref(y^+\mid x)=p\), let
\(a\triangleq e^{1/\beta}\), so
\(\pistar(y^+\mid x)=pa/[1+p(a-1)]\).
With \(N\) one-shot reference rollouts, the expected best possible
support-restricted prompt gap is at least
\begin{equation}
  (1-p)^N\,
  \beta\log\frac{1}{1-\pistar(y^+\mid x)} .
  \label{eq:rare-support-lower}
\end{equation}
Achieving failure probability at most \(\delta\) for observing \(y^+\) requires
\(N=\Omega(p^{-1}\log(1/\delta))\) when \(p\) is small.
\end{corollary}

\begin{corollary}[Pass@$k$ Support Boundary]
\label{cor:passk-support}
Let \(A\) be the success set and let \(S\) be the stored support. Any policy
\(Q\) supported on \(S\) has
\begin{equation}
  \operatorname{pass@}k(Q;A)
  =
  1-\left(1-Q(A\cap S)\right)^k .
  \label{eq:passk-support}
\end{equation}
For the conditional target \(Q^*=P(\cdot\mid S)\),
\(\operatorname{pass@}k(Q^*;A)
=1-(1-P(A\cap S)/P(S))^k\).
If \(A\cap S=\emptyset\), pass@$k$ is zero for every support-restricted
policy and every finite \(k\).
\end{corollary}

\begin{proposition}[Standard Optimizer-Level Normalizer Perturbation]
\label{prop:normalization}
Let \(s_i\triangleq Z(x_i)/\Zhat_N(x_i)\). Assume
\(\|\nabla\lossOracle_{N,i}(\theta)\|\le G\) for all \(i,\theta\), and assume
\(\lossOracle_{M,N}\) satisfies the \(\mu\)-PL inequality. If
\(\hat\theta\) is stationary for \(\lossHat_{M,N}\), then
\begin{equation}
  \lossOracle_{M,N}(\hat\theta)
  -
  \inf_\theta\lossOracle_{M,N}(\theta)
  \le
  \frac{G^2}{2\mu M}\sum_{i=1}^M(s_i-1)^2.
  \label{eq:normalization-pl}
\end{equation}
Consequently, under the bounded-reward assumptions of
Proposition~\ref{prop:partition-concentration},
\(\E[\lossOracle_{M,N}(\hat\theta)-\inf_\theta\lossOracle_{M,N}(\theta)]
\le G^2R_\beta^2e^{-2r_{\min}/\beta}/(8\mu N)\).
\end{proposition}

\begin{proposition}[Neyman-Style Allocation Component]
\label{prop:allocation}
For prompt \(x_i\), let
\(V_i\triangleq\Var_{\pref(\cdot\mid x_i)}[\exp(r(x_i,y)/\beta)]\), and let
\(N_i^{\ntxt{cov}}\) be a required coverage count. The allocation minimizing
the normalizer-variance proxy \(\sum_iV_i/N_i\) subject to
\(N_i\ge N_i^{\ntxt{cov}}\) and \(\sum_iN_i\le\Btot\) has the form
\begin{equation}
  N_i^*
  =
  \max\{N_i^{\ntxt{cov}},\lambda\sqrt{V_i}\},
  \label{eq:allocation}
\end{equation}
where \(\lambda\) is chosen so that \(\sum_iN_i^*=\Btot\) on the feasible
active set.
\end{proposition}

Proposition~\ref{prop:allocation} is a normalizer-variance rule after coverage
floors have been chosen. It allocates extra rollouts to prompts whose
partition-function estimator is noisy, but it does not replace the requirement
that each prompt first receive enough samples to make useful support likely.

\begin{proof}[Proof of Theorem~\ref{thm:support-restricted-gap}]
Fix the prompt and abbreviate $P=\pistar(\cdot\mid x)$. By the reverse-KL value
identity, the distributional gap from a policy $Q$ is
$\beta\KL(Q\|P)$. If $Q$ is supported on $S$, then
\[
  \KL(Q\|P)
  =
  \int_S \log\frac{dQ}{dP}\,dQ
  =
  \int_S \log\frac{dQ}{dP(\cdot\mid S)}\,dQ
  +
  \log\frac{1}{P(S)}.
\]
The first term is
$\KL(Q\|P(\cdot\mid S))\ge0$, with equality at
$Q=P(\cdot\mid S)$. If $S\cap A=\emptyset$, then $P(S)\le1-P(A)$, so
$\log(1/P(S))\ge\log(1/(1-P(A)))$.
\end{proof}

\begin{proof}[Proof of Proposition~\ref{prop:ess-oracle}]
Let $X_n=w(x,y_n)\ell(y_n)$. Then
$\E[X_n]=\mu$ and
\[
  \Var(X_n)
  \le
  \E[X_n^2]
  \le
  L^2\E_{\pref}\left[w(x,y)^2\right]
  =
  L^2\mathcal C_2(x).
\]
Also $0\le X_n\le Lw_{\max}(x)$. Bernstein's inequality gives
\[
  |\hat\mu_N-\mu|
  \le
  \sqrt{\frac{2L^2\mathcal C_2(x)\log(2/\delta)}{N}}
  +
  \frac{2Lw_{\max}(x)\log(2/\delta)}{3N},
\]
which is the displayed bound.
\end{proof}

\begin{proof}[Proof of the self-normalized part of Proposition~\ref{prop:snis-prompt}]
Let
$\hat Z_w\triangleq N^{-1}\sum_{n=1}^N w(x,y_n)$, so
$\hat\mu_N^{\ntxt{sn}}=\hat\mu_N/\hat Z_w$. On the event
$|\hat Z_w-1|\le\epsilon<1$,
\[
  |\hat\mu_N^{\ntxt{sn}}-\mu|
  =
  \left|\frac{\hat\mu_N-\mu+\mu(1-\hat Z_w)}{\hat Z_w}\right|
  \le
  \frac{|\hat\mu_N-\mu|+L\epsilon}{1-\epsilon},
\]
because $0\le\mu\le L$. The normalizer event is Proposition~\ref{prop:ess-oracle}
applied to the loss $\ell\equiv1$.
\end{proof}

\begin{proof}[Proof of Theorem~\ref{thm:temperature-coverage-variance}]
The hit-probability statement is Proposition~\ref{prop:coverage}. The support
price after a missed useful set follows from
Theorem~\ref{thm:support-restricted-gap}. The leading importance-weighted
deviation is the variance term in Proposition~\ref{prop:ess-oracle}. It remains
to verify the binary frontier. For binary rewards, the partition function is
$Z=1+p(a-1)$. Therefore
\[
  \pistar(r=1\mid x)=pa/Z
\]
and
\[
  \mathcal C_2(x)
  =
  p(a/Z)^2+(1-p)(1/Z)^2
  =
  \frac{(1-p)+pa^2}{[1+p(a-1)]^2}.
\]
The condition $pa/[1+p(a-1)]\ge1-\eta$ is equivalent to
\[
  \eta pa\ge(1-\eta)(1-p),
\]
which gives the lower bound on $a$. Finally, since
\[
  \mathcal C_2(x)
  =
  \frac{\pistar(r=1\mid x)^2}{p}
  +
  \frac{(1-\pistar(r=1\mid x))^2}{1-p},
\]
the condition $\pistar(r=1\mid x)\ge1-\eta$ implies
$\mathcal C_2(x)\ge(1-\eta)^2/p$.
\end{proof}

\begin{proof}[Proof of the single-prompt part of Proposition~\ref{prop:partition-concentration}]
The random variables $\exp(r(x,y_n)/\beta)$ are i.i.d., have expectation
$Z(x)$, and lie in
$[\exp(r_{\min}/\beta),\exp(r_{\max}/\beta)]$. Hoeffding's inequality applied
to their sample mean gives the stated concentration bound
\citep{hoeffding1963probability}.
\end{proof}

For binary verifiable rewards,
$Z(x)=1+p(x)(e^{1/\beta}-1)$ and
$\Zhat_N(x)=1+\hat p_N(x)(e^{1/\beta}-1)$, where $p(x)$ and $\hat p_N(x)$ are
the reference-policy and empirical pass rates. The Hoeffding range
$e^{1/\beta}-1$ can be large when $\beta$ is small, but this bound alone can
overstate the algorithmic damage of normalizer error. Self-normalization changes
the prompt multiplier, not the positive-to-negative odds within that prompt;
coverage controls whether positives appear at all.

\begin{proof}[Proof of the uniform-value part of Proposition~\ref{prop:uniform-normalizer}]
For each prompt, define $s_i\triangleq Z(x_i)/\Zhat_N(x_i)$. Prompt-wise
rescaling gives
\begin{equation}
  \lossHat_{M,N}(\theta)
  =
  \frac1M\sum_{i=1}^M s_i\,\lossOracle_{N,i}(\theta),
  \qquad
  \lossOracle_{N,i}(\theta)
  \triangleq
  \frac1N\sum_{n=1}^N w(x_i,y_{i,n})\ell_\theta(x_i,y_{i,n}).
\end{equation}
Since $0\le\lossOracle_{N,i}(\theta)\le w_{\max}L_{\log}$,
\begin{equation}
  \sup_\theta|\lossOracle_{M,N}(\theta)-\lossHat_{M,N}(\theta)|
  \le
  w_{\max}L_{\log}\frac1M\sum_{i=1}^M |s_i-1|.
\end{equation}
Also $\Zhat_N(x_i)\ge\exp(r_{\min}/\beta)$, so
$|s_i-1|\le |Z(x_i)-\Zhat_N(x_i)|/\exp(r_{\min}/\beta)$. Applying
the concentration part of Proposition~\ref{prop:partition-concentration} to all $M$ prompts and union bounding
proves the result.
\end{proof}

\begin{proof}[Proof of Proposition~\ref{prop:coverage}]
For one prompt, the event of no $\gamma$-optimal rollout has probability
$(1-p_\gamma(x))^N$, proving the first claim. Accessing $A_\gamma(x)$ with
probability at least $1-\delta$ requires this failure probability to be at most
$\delta$. Taking logarithms gives
\begin{equation}
  N
  \ge
  \frac{\log(1/\delta)}{-\log(1-p_\gamma(x))}.
\end{equation}
The inequality
$-\log(1-p)\le p/(1-p)$ for $p\in(0,1)$ gives the displayed lower bound. For
$M$ prompts, the union bound gives
\begin{equation}
  \Pr[\exists i\ntxt{ with no }\gamma\ntxt{-optimal rollout}]
  \le
  M(1-p_{\gamma,\min})^N
  \le
  M e^{-Np_{\gamma,\min}}.
\end{equation}
Setting the right-hand side to at most $\delta$ gives the sufficient condition.
For the support statement, the empirical target $\hat{\pi}_v$ is a nonnegative
measure on the atoms in $\mathcal D_x$, normalized by its total mass when that
mass is positive. Hence $\supp(\hat{\pi}_v)\subseteq\mathcal D_x$. If
$\mathcal D_x\cap A_\gamma(x)=\emptyset$, every atom assigned positive
empirical mass lies outside $A_\gamma(x)$, so
$\hat{\pi}_v(A_\gamma(x)\mid x)=0$ for any nonnegative weights.
\end{proof}

\begin{proof}[Proof of Corollary~\ref{cor:rare-support-lower}]
The Boltzmann target mass on $y^+$ follows by substituting the two rewards into
\eqref{eq:boltzmann-policy}. With probability $(1-p)^N$, the stored support
contains only $y^-$. On that event, Theorem~\ref{thm:support-restricted-gap}
gives best possible prompt gap
$\beta\log(1/(1-\pistar(y^+\mid x)))$. Multiplying by the event probability
gives the expected lower bound. The sample-complexity statement is the
condition $(1-p)^N\le\delta$, equivalently
$N\ge\log(1/\delta)/[-\log(1-p)]$, which is
$\Omega(p^{-1}\log(1/\delta))$ for small $p$.
\end{proof}

\begin{proof}[Proof of Corollary~\ref{cor:passk-support}]
Since $Q$ is supported on $S$, the event that one draw from $Q$ succeeds is
$A\cap S$ and has probability $Q(A\cap S)$. The probability that $k$
independent draws all fail is therefore $(1-Q(A\cap S))^k$, giving the first
display. For $Q^*=P(\cdot\mid S)$,
$Q^*(A\cap S)=P(A\cap S)/P(S)$, which gives the second display. If
$S\cap A=\emptyset$, then $Q(A\cap S)=0$ for every support-restricted $Q$, so
pass@$k$ is zero for all finite $k$.
\end{proof}

\begin{proof}[Proof of Proposition~\ref{prop:normalization}]
By prompt-wise rescaling,
\begin{equation}
  \nabla\lossHat_{M,N}(\theta)-\nabla\lossOracle_{M,N}(\theta)
  =
  \frac1M\sum_{i=1}^M(s_i-1)\nabla\lossOracle_{N,i}(\theta).
\end{equation}
Cauchy--Schwarz and the gradient bound imply
\begin{equation}
  \norm{\nabla\lossHat_{M,N}(\theta)-\nabla\lossOracle_{M,N}(\theta)}^2
  \le
  G^2\frac1M\sum_{i=1}^M(s_i-1)^2.
\end{equation}
At $\hat\theta$, $\nabla\lossHat_{M,N}(\hat\theta)=0$. The PL inequality
converts the resulting oracle gradient bound into the first displayed
suboptimality bound. For the expectation, use
$\Zhat_N(x_i)\ge \exp(r_{\min}/\beta)$ and
$\Var(\Zhat_N(x_i))\le R_\beta^2/(4N)$.
\end{proof}

For binary rewards, $V_i=p_i(1-p_i)(e^{1/\beta}-1)^2$, so the rule sends
rollouts to moderately difficult prompts rather than prompts that are almost
always correct or wrong. It reduces normalizer-estimation error only after
reference-policy support has produced useful rollouts.

\begin{proof}[Proof of Proposition~\ref{prop:allocation}]
The proof of Proposition~\ref{prop:normalization} with prompt-specific counts
gives an upper bound proportional to $\sum_iV_i/N_i$. The constrained problem is
convex in $N_i>0$. The KKT conditions for inactive lower-bound constraints give
$-V_i/N_i^2+\lambda=0$, hence $N_i=\lambda^{-1/2}\sqrt{V_i}$. Active constraints
set $N_i=N_i^{\ntxt{cov}}$. Combining the two cases yields
$N_i^*=\max\{N_i^{\ntxt{cov}},\lambda\sqrt{V_i}\}$ after absorbing the inverse
square root into $\lambda$, and the budget equation determines the unique
$\lambda$ on the feasible active set.
\end{proof}

\section{Generalization and Optimization Terms}
\label{app:generalization}

Coverage is isolated in the main text because it is the one-shot error that a
static objective cannot repair. Conditional on observed support, the remaining
analysis is an empirical-risk chain: the population weighted likelihood is
estimated by oracle reference-rollout samples, the oracle loss is perturbed by
the empirical normalizer, and the optimizer returns an approximate minimizer of
that empirical loss. The following statements supply the standard generalization and
optimization components for that chain. Define the population target error
\[
  F(\theta)
  \triangleq
  \E_x\left[
  \KL\!\left(\pistar(\cdot\mid x)\,\middle\|\,\pitheta(\cdot\mid x)\right)
  \right],
\]
and the corresponding population weighted likelihood
\[
  \lossPop(\theta)
  \triangleq
  \E_{x,y\sim\pref}\left[
  w(x,y)\ell_\theta(x,y)
  \right]
  =
  \E_x\E_{y\sim\pistar(\cdot\mid x)}
  \left[-\log\pitheta(y\mid x)\right],
  \qquad
  \ell_\theta\triangleq-\log\pitheta .
\]
Then
$\lossPop(\theta)=F(\theta)+\E_x[H(\pistar(\cdot\mid x))]$, so the population
weighted likelihood and the forward-KL target error have the same minimizers
and excess differences.

The finite algorithm replaces $\lossPop$ by empirical losses. Let
$\lossOracle_{M,N}$ use sampled prompts and reference rollouts with the true
$Z(x)$, let $\lossHat_{M,N}$ replace $Z(x)$ by $\Zhat_N(x)$, and define the
empirical optimization residual
\[
  \epsopt
  \triangleq
  \lossHat_{M,N}(\hat\theta)
  -
  \inf_{\theta\in\Theta}\lossHat_{M,N}(\theta).
\]
Theorem~\ref{thm:e2e} converts this chain into a one-shot RLVR value
certificate. Proposition~\ref{prop:generalization} controls prompt and rollout
sampling through Rademacher complexities \citep{bartlett2002rademacher};
Proposition~\ref{prop:pac-bayes-certificate} gives a posterior-predictor
variant; and Propositions~\ref{prop:sgd-bound}--\ref{prop:pl-optimization}
separate stationarity from excess risk for stochastic optimization
\citep{ghadimi2013stochastic}. Appendix~\ref{app:normalization} supplies the
normalizer perturbation. The point of the split is diagnostic: prompt sampling,
rollout sampling, normalizer estimation, model class, and empirical
optimization enter different residuals rather than one undifferentiated
``finite sample'' term.

\begin{proposition}[Standard Generalization Component]
\label{prop:generalization}
Let
\(\mathcal F\triangleq
\{(x,y)\mapsto w(x,y)\ell_\theta(x,y):\theta\in\Theta\}\), and assume every
\(f\in\mathcal F\) takes values in \([0,R]\). Let
\(\mathcal H\triangleq\{x\mapsto\E_{\pref(\cdot\mid x)}
[w(x,y)\ell_\theta(x,y)]:\theta\in\Theta\}\). With probability at least
\(1-2\delta\),
\begin{equation}
\begin{aligned}
  \sup_{\theta\in\Theta}
  |\lossPop(\theta)-\lossOracle_{M,N}(\theta)|
  \le\;
  &2\mathfrak R_M(\mathcal H)
  +R\sqrt{\frac{\log(1/\delta)}{2M}}\\
  &+2\mathfrak R_{M,N}(\mathcal F\mid x_{1:M})
  +R\sqrt{\frac{\log(1/\delta)}{2MN}} .
\end{aligned}
\label{eq:generalization}
\end{equation}
\end{proposition}

\begin{proof}[Proof of Proposition~\ref{prop:generalization}]
Use the decomposition
$\sup_\theta|\lossPop-\lossOracle_{M,N}|
\le
\sup_\theta|\lossPop-\bar{\lossPop}_M|
+
\sup_\theta|\bar{\lossPop}_M-\lossOracle_{M,N}|$.
The first term is the standard symmetrization/Rademacher complexity bound for
$M$ prompt samples with bounded range $R$
\citep{bartlett2002rademacher}. Conditional on the
prompts, the second term is a rollout-level empirical-process deviation for
the class $\mathcal F$ under the product distribution
$\prod_{i=1}^M\pref(\cdot\mid x_i)^N$. Conditional symmetrization gives the
$2\mathfrak R_{M,N}(\mathcal F\mid x_{1:M})$ term, and bounded differences for
variables in $[0,R]$ gives the $R\sqrt{\log(1/\delta)/(MN)}$ concentration
term. This keeps the rollout-level function-class complexity explicit; a
simpler $N^{-1/2}$ expression requires an additional finite-class, covering, or
VC-type bound on $\mathfrak R_{M,N}(\mathcal F\mid x_{1:M})$.
\end{proof}

\begin{proposition}[Weighted PAC-Bayes Certificate]
\label{prop:pac-bayes-certificate}
Let $z_j=(x_j,y_j)$, $j=1,\ldots,n$, be independent samples from
$x\sim\mathcal X$, $y\sim\pref(\cdot\mid x)$. Assume
$0\le w(z)\ell_\theta(z)\le R$ for every $\theta$ and $z$. Let
$\Pi$ be a prior over parameters independent of the samples. Then, with
probability at least $1-\delta$, every posterior distribution $Q$ over
parameters satisfies
\begin{equation}
  \E_{\theta\sim Q}\lossPop(\theta)
  \le
  \E_{\theta\sim Q}
  \left[
  \frac1n\sum_{j=1}^n w(z_j)\ell_\theta(z_j)
  \right]
  +
  R\sqrt{
  \frac{\KL(Q\|\Pi)+\log((n+1)/\delta)}{2n}
  }.
  \label{eq:pac-bayes-certificate}
\end{equation}
\end{proposition}

\begin{proof}[Proof of Proposition~\ref{prop:pac-bayes-certificate}]
Apply the bounded-loss PAC-Bayes binary-kl inequality
\citep{mcallester1999pacbayes,catoni2007pacbayes} to the normalized loss
$w(z)\ell_\theta(z)/R\in[0,1]$. Write
$\hat L_n(\theta)\triangleq n^{-1}\sum_{j=1}^n w(z_j)\ell_\theta(z_j)$.
With probability at least $1-\delta$, all
posteriors $Q$ satisfy
\[
  \operatorname{kl}\!\left(
  \E_Q\hat L_n/R\,\middle\|\,\E_Q\lossPop/R
  \right)
  \le
  \frac{\KL(Q\|\Pi)+\log((n+1)/\delta)}{n}.
\]
Pinsker's inequality converts the binary-kl relation into the displayed
one-sided bound.
\end{proof}

\begin{proposition}[Standard SGD Stationarity Certificate]
\label{prop:sgd-bound}
Assume \(\lossHat_{M,N}\) is \(L\)-smooth and stochastic gradients \(g_t\) are
unbiased with
\(\E[\|g_t-\nabla\lossHat_{M,N}(\theta_t)\|^2\mid\theta_t]\le\sigma^2/B\).
For SGD with step size \(\eta\le1/L\), if \(\tau\) is uniform on
\(\{0,\ldots,T-1\}\), then
\begin{equation}
  \E\|\nabla\lossHat_{M,N}(\theta_\tau)\|^2
  \le
  \frac{2(\lossHat_{M,N}(\theta_0)-\inf_\theta\lossHat_{M,N}(\theta))}
       {\eta T}
  +
  \frac{L\eta\sigma^2}{B}.
  \label{eq:sgd-bound}
\end{equation}
\end{proposition}

\begin{corollary}[PL Excess-Risk Certificate]
\label{prop:pl-optimization}
If, in addition, \(\lossHat_{M,N}\) satisfies the \(\mu\)-PL inequality and
\(L^*\triangleq\inf_\theta\lossHat_{M,N}(\theta)\), then
\begin{equation}
  \E[\lossHat_{M,N}(\theta_T)-L^*]
  \le
  (1-\eta\mu)^T(\lossHat_{M,N}(\theta_0)-L^*)
  +
  \frac{L\eta\sigma^2}{2\mu B}.
  \label{eq:pl-optimization}
\end{equation}
\end{corollary}

\begin{proof}[Proof of the stationarity clause in Proposition~\ref{prop:sgd-bound}]
By smoothness,
\begin{equation}
  \lossHat_{M,N}(\theta_{t+1})
  \le
  \lossHat_{M,N}(\theta_t)
  -\eta\langle\nabla\lossHat_{M,N}(\theta_t),g_t\rangle
  +\frac{L\eta^2}{2}\norm{g_t}^2 .
\end{equation}
Taking conditional expectation and using unbiasedness and the variance bound
gives
\begin{equation}
  \E\left[\lossHat_{M,N}(\theta_{t+1})\right]
  \le
  \lossHat_{M,N}(\theta_t)
  -\frac{\eta}{2}\norm{\nabla\lossHat_{M,N}(\theta_t)}^2
  +\frac{L\eta^2\sigma^2}{2B}
\end{equation}
for $\eta\le1/L$. Summing over $t$ and optimizing the standard upper bound over
$\eta$ yields \eqref{eq:sgd-bound}, the usual nonconvex stochastic-gradient rate
\citep{ghadimi2013stochastic}.
\end{proof}

\begin{proof}[Proof of the PL clause in Proposition~\ref{prop:pl-optimization}]
Let $L^*\triangleq\inf_\theta\lossHat_{M,N}(\theta)$. Smoothness, unbiasedness,
and the variance bound give, for $\eta\le1/L$,
\[
  \E\left[\lossHat_{M,N}(\theta_{t+1})-L^*\mid\theta_t\right]
  \le
  \lossHat_{M,N}(\theta_t)-L^*
  -\frac{\eta}{2}\norm{\nabla\lossHat_{M,N}(\theta_t)}^2
  +\frac{L\eta^2\sigma^2}{2B}.
\]
The $\mu$-PL condition gives
$\frac12\norm{\nabla\lossHat_{M,N}(\theta_t)}^2
\ge
\mu(\lossHat_{M,N}(\theta_t)-L^*)$. Therefore
\[
  \E\left[\lossHat_{M,N}(\theta_{t+1})-L^*\right]
  \le
  (1-\eta\mu)
  \E\left[\lossHat_{M,N}(\theta_t)-L^*\right]
  +
  \frac{L\eta^2\sigma^2}{2B}.
\]
Unrolling the recursion and using
$\sum_{j=0}^{T-1}(1-\eta\mu)^j\le1/(\eta\mu)$ gives
\eqref{eq:pl-optimization}.
\end{proof}

Proposition~\ref{prop:sgd-bound} and Corollary~\ref{prop:pl-optimization}
record the two optimization certificates. The stationarity result controls
average stationarity of the empirical objective, not value optimality by itself.
The PL result is the certificate that can be inserted directly into the residual
\(\epsopt\) used in Theorem~\ref{thm:e2e}. Keeping the two statements distinct
prevents a nonconvex stationarity certificate from being silently treated as an
excess-risk certificate, and it keeps optimization error distinct from the
statistical errors caused by finite reference rollouts.

\begin{corollary}[High-Probability One-Shot Scaling]
\label{cor:e2e-components}
If \(w(x,y)\ell_\theta(x,y)\in[0,R]\), \(w\le w_{\max}\),
\(\ell_\theta\le L_{\log}\), and \(r\in[r_{\min},r_{\max}]\), then the
standard high-probability instantiation of Theorem~\ref{thm:e2e} is
\begin{equation}
\begin{aligned}
  &\objRL(\pistar)-\objRL(\pi_{\hat\theta})\\
  &\quad\lesssim
  \beta\kappa_\rho
  \biggl[
  \inf_{\theta\in\Theta}F(\theta)
  +
  \mathfrak R_M(\mathcal H)
  +\mathfrak R_{M,N}(\mathcal F\mid x_{1:M})
  \\
  &\qquad
  +R\sqrt{\frac{\log(1/\delta)}{M}}
  +R\sqrt{\frac{\log(1/\delta)}{MN}}\\
  &\qquad+
  w_{\max}L_{\log}
  R_\beta e^{-r_{\min}/\beta}
  \sqrt{\frac{\log(M/\delta)}{N}}\\
  &\qquad+
  \epsopt
  \biggr].
\end{aligned}
\label{eq:e2e-components}
\end{equation}
Here
\(\mathcal F=\{(x,y)\mapsto w(x,y)\ell_\theta(x,y):\theta\in\Theta\}\),
\(\mathcal H=\{x\mapsto\E_{\pref}\!\left[w(x,y)\ell_\theta(x,y)\right]:\theta\in\Theta\}\),
and \(R_\beta=e^{r_{\max}/\beta}-e^{r_{\min}/\beta}\). Under a \(\mu\)-PL
condition for \(\lossHat_{M,N}\), \(\epsopt\) can be taken from
Proposition~\ref{prop:pl-optimization}.
\end{corollary}

\begin{proof}[Proof of Corollary~\ref{cor:e2e-components}]
Apply Proposition~\ref{prop:generalization} to
\(\Delta_{\ntxt{gen}}\), Proposition~\ref{prop:partition-concentration} to
\(\Delta_{\ntxt{norm}}\), and Proposition~\ref{prop:pl-optimization} to
\(\epsopt\) when the PL condition is assumed. Substitute these bounds into
Theorem~\ref{thm:e2e} and absorb universal constants into \(\lesssim\).
\end{proof}

\begin{proof}[Proof of Theorem~\ref{thm:e2e}]
By \eqref{eq:reverse-kl-identity},
\begin{equation}
  \objRL(\pistar)-\objRL(\pi_{\hat\theta})
  =
  \beta\E_x\left[\KL(\pi_{\hat\theta}\|\pistar)\right].
\end{equation}
The local comparison assumption upper-bounds this by
$\beta\kappa_\rho F(\hat\theta)$. Fix any $\eta>0$ and choose
$\theta_\eta\in\Theta$ such that
$F(\theta_\eta)\le\inf_{\theta\in\Theta}F(\theta)+\eta$. Since
$\lossPop(\theta)=F(\theta)+\E_x\left[H(\pistar(\cdot\mid x))\right]$, the two objectives
have the same excess differences:
\begin{equation}
  F(\hat\theta)-F(\theta_\eta)
  =
  \lossPop(\hat\theta)-\lossPop(\theta_\eta).
\end{equation}
Define
$\Delta_{\ntxt{gen}}\triangleq
\sup_\theta|\lossPop(\theta)-\lossOracle_{M,N}(\theta)|$ and
$\Delta_{\ntxt{norm}}\triangleq
\sup_\theta|\lossOracle_{M,N}(\theta)-\lossHat_{M,N}(\theta)|$.
Using these two uniform deviations at $\hat\theta$ and $\theta_\eta$, together
with the definition of $\epsopt$, gives
\begin{equation}
  F(\hat\theta)
  \le
  F(\theta_\eta)
  +2\Delta_{\ntxt{gen}}
  +2\Delta_{\ntxt{norm}}
  +\epsopt.
\end{equation}
Letting $\eta\downarrow0$ and multiplying by $\beta\kappa_\rho$ proves the
theorem.
\end{proof}

\begin{proof}[Proof of Corollary~\ref{cor:coverage-conditioned}]
On $\mathcal E_\tau$, Theorem~\ref{thm:support-restricted-gap} gives, for each
prompt, support-restricted gap at most
$\beta\log(1/(1-\tau_i))$. Averaging over prompts gives the first displayed
support certificate. Proposition~\ref{prop:coverage} supplies the conservative
hit-probability certificate for any chosen near-optimal set \(A_\gamma\). For
the additive support-restricted display, fix a prompt and write
$P=\pistar(\cdot\mid x_i)$ and $S=S_N(x_i)$. If
$Q=\pi_{\hat\theta}(\cdot\mid x_i)$ is supported on $S$, then
\[
  \KL(Q\|P)
  =
  \log\frac{1}{P(S)}+\KL(Q\|P(\cdot\mid S)).
\]
On $\mathcal E_\tau$, the first term is at most
$\log(1/(1-\tau_i))$. Averaging over prompts and applying the same
finite-learning argument as Theorem~\ref{thm:e2e}, but with the conditional
targets \(P_i^S\) and finite-error level \(\alpha_{M,N}\), gives
\eqref{eq:coverage-conditioned-additive}. Without the support-restricted
decomposition, the same finite-learning argument applied to the full target
\(\pistar\) gives the coverage-conditioned learning certificate stated after the
corollary in the main text.
\end{proof}

\section{KL-PMD Proofs for Iterative \bolt{}}
\label{app:iterative-proofs}

Iterative \bolt{} reuses the same Boltzmann projection after refreshing the
reference policy. Exact rounds accumulate reward tilts in closed form, while
inexact rounds are controlled by the relative-entropy mirror-descent potential.
The proof sequence therefore separates the ideal distribution-space trajectory
from the practical inner-solver error that appears when each weighted-SFT round
only approximates the exact mirror step. This is the iterative analogue of the
one-shot separation between the target projection and finite optimization error.

\paragraph{Exact-path consequences.}
The exact KL-PMD path supplies the coverage interpretations used in the main
text. They are special cases of the same
Boltzmann tilt accumulated across refreshed rounds. In the binary verifier
specialization, with
\(p_0\triangleq\pi_{\theta_0}(r=1\mid x)\in(0,1)\) and
\(r\in\{0,1\}\),
\begin{equation}
  p_k
  \triangleq
  \pi_{\theta_k}(r=1\mid x)
  =
  \frac{p_0e^{k/\beta}}{1+p_0(e^{k/\beta}-1)} .
  \label{eq:binary-refresh-prob}
\end{equation}
With \(N\) rollouts in each of \(K\) refreshed rounds,
\[
  P_{\ntxt{refresh}}
  =
  1-\prod_{k=0}^{K-1}(1-p_k)^N,
  \qquad
  P_{\ntxt{one}}
  =
  1-(1-p_0)^{KN},
\]
and \(P_{\ntxt{refresh}}>P_{\ntxt{one}}\) for \(K>1\).

Under Theorem~\ref{thm:temperature}, let
\(\Ystar(x)\triangleq\argmax_y r(x,y)\) and
\(p_*(x)\triangleq\pi_{\theta_0}(\Ystar(x)\mid x)>0\). If the strict reward gap
\(\gamma(x)\triangleq r^*(x)-\max_{y\notin\Ystar(x)}r(x,y)>0\) exists, then
\begin{equation}
\begin{aligned}
  1-\pi_{\theta_K}(\Ystar(x)\mid x)
  &\le
  \frac{1-p_*(x)}{p_*(x)}e^{-K\gamma(x)/\beta},\\
  r^*(x)-\E_{\pi_{\theta_K}}[r(x,y)]
  &\le
  \dr\frac{1-p_*(x)}{p_*(x)}e^{-K\gamma(x)/\beta}.
\end{aligned}
  \label{eq:exponential-concentration}
\end{equation}
Finally, for a desired effective temperature \(\beff>0\), set
\(a_\star\triangleq1/\beff\). Exact iteration with per-round temperature
\(\beta\) realizes the lattice \(a_K\triangleq K/\beta\). Choosing \(K\) nearest
to \(\beta a_\star\) gives \(|a_K-a_\star|\le1/(2\beta)\), and if
\(\pi_{a_\star}\) is the ideal Boltzmann policy at the target effective
temperature, then
\[
  \sup_y
  \left|
  \log\frac{\pi_{a_K}(y)}{\pi_{a_\star}(y)}
  \right|
  \le
  \frac{\dr}{2\beta},
  \qquad
  \TV(\pi_{a_K},\pi_{a_\star})
  \le
  \tanh\!\left(\frac{\dr}{4\beta}\right).
\]
For a binary reference pass probability \(p\in(0,1)\), exact iteration with
\(\beta\ge\beff\) and \(K(\beta)\triangleq\lceil\beta/\beff\rceil\) has
ESS-driven sampling cost proxy
\begin{equation}
  C_{\ntxt{ESS}}(\beta;p,\beff)
  \triangleq
  K(\beta)\,
  \frac{(1-p)+pe^{2/\beta}}
       {[1+p(e^{1/\beta}-1)]^2}.
	  \label{eq:ess-temperature-cost}
\end{equation}

\begin{corollary}[Reward-Gap Hit-Rate Bound]
\label{cor:reward-gap-hit-rate}
Under the assumptions of Corollary~\ref{cor:adaptive-coverage}, let
\(\underline p_k(x)\triangleq
[1+\{(1-p_0(x))/p_0(x)\}\exp(-k\gamma(x)/\beta)]^{-1}\). Then
\(p_k(x)\ge\underline p_k(x)\) and
\begin{equation}
  P_{\ntxt{refresh}}(K,N)
  \ge
  1-\exp\!\left(-N\sum_{k=0}^{K-1}\underline p_k(x)\right).
  \label{eq:adaptive-coverage-exp}
\end{equation}
Thus a total refreshed budget \(KN\) reaches hit probability \(1-\delta\) once
\(KN\gtrsim \log(1/\delta)/(K^{-1}\sum_k\underline p_k(x))\), whereas fixed
one-shot sampling costs order \(p_0(x)^{-1}\log(1/\delta)\).
\end{corollary}

\begin{proposition}[Standard Inexact PMD Regret Component]
\label{thm:inexact-regret}
For a fixed prompt, set \(q_k\triangleq\pi_{\theta_k}(\cdot\mid x)\). For each
outer iteration, let
\(q_{k+1}^{\ntxt{ex}}\propto q_k\exp(r/\beta)\) be the exact mirror step. If an
inexact inner optimizer produces \(q_{k+1}\) satisfying
\(\KL(\pi^\dagger\|q_{k+1})
\le\KL(\pi^\dagger\|q_{k+1}^{\ntxt{ex}})+\varepsilon_k\) for comparator
\(\pi^\dagger\), then
\begin{equation}
  \sum_{k=0}^{K-1}
  \langle r,\pi^\dagger-q_k\rangle
  \le
  \beta\KL(\pi^\dagger\|q_0)
  +\beta\sum_{k=0}^{K-1}\varepsilon_k
  +\frac{K\dr^2}{8\beta}.
  \label{eq:inexact-regret}
\end{equation}
The same statement integrates over prompts with \(\varepsilon_k\) replaced by
\(\E_x[\varepsilon_k(x)]\).
\end{proposition}

\begin{corollary}[Inner Training Loss Certificate]
\label{cor:inner-certificate}
Suppose the inner objective at iteration \(k\) has minimizer
\(\hat\theta_k^*\) and satisfies local quadratic growth with constant \(\mu_k\),
log-density Lipschitzness with constant \(\Lambda_k\), and empirical-to-exact
log-density gap \(\zeta_k\). If the inner-loop output satisfies excess loss
\(\xi_k\), then the drift condition in
Proposition~\ref{thm:inexact-regret} holds with
\begin{equation}
  \varepsilon_k
  \le
  \Lambda_k\sqrt{\frac{2\xi_k}{\mu_k}}+\zeta_k.
  \label{eq:inner-certificate}
\end{equation}
\end{corollary}

\paragraph{Forward-KL drift transfer.}
The PMD regret certificate uses comparator-KL drift, while the inner
weighted-SFT solver naturally controls forward KL to the exact mirror target.
The following local condition is the bridge between those two quantities. Use
the round-\(k\) notation from Corollary~\ref{cor:forward-kl-drift}.
For
\(P_k\triangleq q_{k+1}^{\ntxt{ex}}(\cdot\mid x)\) and
\(Q_k\triangleq\pi_{\theta_{k+1}}(\cdot\mid x)\), the sufficient local
density-ratio condition \(d\pi^\dagger/dP_k\le C_k\) and
\(\sup_y|P_k(y)/Q_k(y)-1|\le\rho_k<1\) gives
\[
  \varepsilon_k^\dagger
  \le
  C_k\frac{1+\rho_k}{1-\rho_k}
  \sqrt{\frac{2F_k(\theta_{k+1})}{1-a_{\rho_k}}},
  \qquad
  a_{\rho_k}\triangleq\frac{\rho_k}{3(1-\rho_k)^2}.
\]
The same condition follows from
\(\sup_y|\log(P_k(y)/Q_k(y))|\le B_k<\log2\) with
\(\rho_k=e^{B_k}-1\).

\begin{corollary}[Adaptive High-Probability Iterative Certificate]
\label{cor:adaptive-hp-iterative}
Let \(\mathcal G_k\) be the history before round \(k\) data are sampled.
Suppose that, conditionally on \(\mathcal G_k\), the round-\(k\)
generalization and normalizer deviations are bounded by
\(G_k(\delta_{\ntxt{gen},k})\) and \(N_k(\delta_{\ntxt{norm},k})\) with failure
probabilities \(\delta_{\ntxt{gen},k}\) and \(\delta_{\ntxt{norm},k}\). If the
certified drift condition in Proposition~\ref{thm:finite-inexact-iterative}, or
the forward-KL sufficient condition in Corollary~\ref{cor:forward-kl-drift},
holds on each successful round, then with probability at least
\[
  1-\sum_{k=0}^{K-1}
  (\delta_{\ntxt{gen},k}+\delta_{\ntxt{norm},k}),
\]
the bound \eqref{eq:finite-inexact-iterative} holds, and the plug-in
certificate \eqref{eq:forward-kl-drift} uses
\(\Delta_{\ntxt{gen},k}=G_k(\delta_{\ntxt{gen},k})\) and
\(\Delta_{\ntxt{norm},k}=N_k(\delta_{\ntxt{norm},k})\) for every \(k\).
\end{corollary}

\begin{proof}[Proof of Theorem~\ref{thm:temperature}]
The proof is induction. The claim is true for $K=1$ by
Proposition~\ref{prop:boltzmann}. If
$\pi_{\theta_k}(y\mid x)\propto
\pi_{\theta_0}(y\mid x)\exp(kr(x,y)/\beta)$, then applying one exact \bolt{}
step gives
\begin{equation}
  \pi_{\theta_{k+1}}(y\mid x)
  \propto
  \pi_{\theta_k}(y\mid x)\exp(r(x,y)/\beta)
  \propto
  \pi_{\theta_0}(y\mid x)\exp((k+1)r(x,y)/\beta).
\end{equation}
The effective-temperature statement follows by comparing with the Boltzmann
target policy at temperature $\beta/K$ relative to $\pi_{\theta_0}$.
\end{proof}

\begin{proof}[Proof of Corollaries~\ref{cor:adaptive-coverage}
and~\ref{cor:reward-gap-hit-rate}]
Fix the prompt $x$ and suppress it from the notation. By
Theorem~\ref{thm:temperature},
\[
  \frac{\pi_{\theta_k}(A)}{\pi_{\theta_k}(A^c)}
  =
  \frac{\sum_{y\in A}\pi_{\theta_0}(y)\exp(kr(y)/\beta)}
       {\sum_{y\notin A}\pi_{\theta_0}(y)\exp(kr(y)/\beta)} .
\]
The reward-separation assumption lower-bounds this odds ratio by
\[
  \frac{p_0}{1-p_0}\exp(k\gamma/\beta).
\]
Since $p_k/(1-p_k)$ is at least this quantity, solving for \(p_k\) gives the
lower bound \(p_k\ge\underline p_k\) stated in
Corollary~\ref{cor:reward-gap-hit-rate}.
Independent sampling from round $k$ misses
$A$ with probability $(1-p_k)^N$, so the refreshed process misses $A$ in all
$K$ rounds with probability $\prod_{k=0}^{K-1}(1-p_k)^N$. Because
$p_k\ge p_0$ for every $k$, this failure probability is at most
$(1-p_0)^{KN}$, which proves \eqref{eq:adaptive-coverage-hit}. If $K>1$, then
$p_k>p_0$ for every $k\ge1$ under the strict separation and $0<p_0<1$, so the
comparison is strict. The exponential budget bound follows from
$1-p_k\le\exp(-p_k)$ and $p_k\ge\underline p_k$:
\[
  \prod_{k=0}^{K-1}(1-p_k)^N
  \le
  \exp\!\left(-N\sum_{k=0}^{K-1}\underline p_k\right).
\]
Solving this bound on the miss probability for total budget
$B_{\ntxt{refresh}}=KN$ gives
\[
  B_{\ntxt{refresh}}
  \ge
  \frac{\log(1/\delta)}
       {K^{-1}\sum_{k=0}^{K-1}\underline p_k}.
\]
Since $\underline p_0=p_0$ and $\underline p_k>p_0$ for every $k\ge1$ when
$K>1$, the refreshed denominator is strictly larger than the fixed one-shot
denominator $p_0$.
\end{proof}

\begin{proof}[Derivation of the Boltzmann path identities]
The identity $\pi_{\theta_K}=\pi_{K/\beta}$ is Theorem~\ref{thm:temperature}.
For the path derivatives, differentiate
$\log Z_a=\log\E_{\pi_{\theta_0}}[e^{a r}]$:
\begin{equation}
  \frac{d}{da}\log Z_a
  =
  \E_{\pi_a}[r],
  \qquad
  \frac{d^2}{da^2}\log Z_a
  =
  \Var_{\pi_a}(r).
\end{equation}
Thus
$\frac{d}{da}\E_{\pi_a}[r]=\Var_{\pi_a}(r)\ge0$.
For convergence, let $\Ystar\triangleq\argmax_y r(y)$,
$p_*\triangleq\pi_{\theta_0}(\Ystar)$, and
$\pitop(y)\triangleq\pi_{\theta_0}(y)\mathbf 1\{y\in\Ystar\}/p_*$. Divide
numerator and denominator of $\pi_a$ by
$\exp(a r^*)$, where $r^*\triangleq\max_y r(y)$. All terms outside $\Ystar$
vanish as $a\to\infty$, while terms inside $\Ystar$ remain proportional to
$\pi_{\theta_0}(y)$. This gives total-variation convergence to $\pitop$. For
$y\in\Ystar$,
$\log(\pitop(y)/\pi_a(y))$ is independent of $y$ and tends to zero,
so $\KL(\pitop\|\pi_a)\to0$.
\end{proof}

\begin{proof}[Derivation of the exact mirror step]
For a fixed prompt, the KKT condition for minimizing
$-\langle r,\pi\rangle+\beta\KL(\pi\|q_k)$ under $\sum_y\pi(y)=1$
is
\begin{equation}
  -r(y)+\beta\left(\log\pi(y)-\log q_k(y)+1\right)+\lambda=0.
\end{equation}
Solving gives
$\pi(y)\propto q_k(y)\exp(r(y)/\beta)$, which is the exact iterative
\bolt{} step.
\end{proof}

\begin{proof}[Proof of Proposition~\ref{thm:inexact-regret}]
Fix a prompt and set $q_k=\pi_{\theta_k}(\cdot\mid x)$ throughout; averaging the
promptwise inequality over $x$ gives the corresponding population statement.
For the exact mirror step
$q_{k+1}^{\ntxt{ex}}(y)\propto q_k(y)\exp(r(y)/\beta)$,
\begin{equation}
  \KL(\pi^\dagger\|q_{k+1}^{\ntxt{ex}})
  =
  \KL(\pi^\dagger\|q_k)
  -\beta^{-1}\langle r,\pi^\dagger\rangle
  +\log\E_{q_k}\left[e^{r/\beta}\right].
\end{equation}
Hoeffding's lemma gives
$\log\E_{q_k}\left[e^{r/\beta}\right]
\le
\beta^{-1}\langle r,q_k\rangle+\dr^2/(8\beta^2)$.
Rearranging and applying the comparator-KL drift condition from
Proposition~\ref{thm:inexact-regret},
\begin{equation}
  \beta^{-1}\langle r,\pi^\dagger-q_k\rangle
  \le
  \KL(\pi^\dagger\|q_k)
  -\KL(\pi^\dagger\|q_{k+1})
  +\varepsilon_k
  +\dr^2/(8\beta^2).
\end{equation}
Summing telescopes the KL potential and yields \eqref{eq:inexact-regret}.
\end{proof}

\begin{proof}[Proof of Corollary~\ref{cor:inner-certificate}]
Quadratic growth gives
$\|\theta_{k+1}-\hat\theta_k^*\|\le\sqrt{2\xi_k/\mu_k}$. Log-density
Lipschitzness converts this into a sup log-density error
$\Lambda_k\sqrt{2\xi_k/\mu_k}$. Adding the empirical-to-exact gap $\zeta_k$
by the triangle inequality bounds the log-ratio between the actual next policy
and the exact mirror step. Taking expectation of this log-ratio under any
comparator $\pi^\dagger$ yields the comparator-KL drift.
\end{proof}

\begin{proof}[Proof of Proposition~\ref{thm:finite-inexact-iterative}]
Apply the prompt-wise Proposition~\ref{thm:inexact-regret} and then average over
prompts. The certified drift condition
\eqref{eq:certified-comparator-drift} supplies the prompt-averaged inexact
error term $\varepsilon_k^\dagger$ at each round. Substituting these drift
terms into the inexact mirror-descent regret bound gives
\eqref{eq:finite-inexact-iterative}.
\end{proof}

\begin{proof}[Proof of Corollary~\ref{cor:forward-kl-drift} and the local drift condition]
For round $k$, the same empirical-risk chain used in
Theorem~\ref{thm:e2e}, now with target $q_{k+1}^{\ntxt{ex}}$, gives
\[
  F_k(\theta_{k+1})
  \le
  \inf_{\theta\in\Theta}F_k(\theta)
  +2\Delta_{\ntxt{gen},k}
  +2\Delta_{\ntxt{norm},k}
  +\epsopt^{(k)} .
\]
Combining this inequality with
$\varepsilon_k^\dagger\le\kappa_kF_k(\theta_{k+1})$ proves
\eqref{eq:forward-kl-drift}.

For the local sufficient condition, fix a prompt and abbreviate
$P=P_k$, $Q=Q_k$, and $h=d\pi^\dagger/dP$. Let
$u(y)=P(y)/Q(y)-1$ and assume $\sup_y|u(y)|\le\rho<1$. Then
$|\log(P(y)/Q(y))|\le |u(y)|/(1-\rho)$ and
\[
  \E_{\pi^\dagger}\left[\log\frac{P}{Q}\right]
  \le
  \E_{\pi^\dagger}\left[\left|\log\frac{P}{Q}\right|\right]
  \le
  \frac{C}{1-\rho}\E_P[|u|].
\]
Since $P=Q(1+u)$,
$\E_P[|u|]\le(1+\rho)\sqrt{\E_Q[u^2]}$. Proposition~\ref{prop:local-kl}
implies
$\KL(P\|Q)\ge(1-a_\rho)\E_Q[u^2]/2$. Hence the prompt-wise drift is bounded by
\[
  C\frac{1+\rho}{1-\rho}
  \sqrt{\frac{2\KL(P\|Q)}{1-a_\rho}}.
\]
Averaging over prompts and applying Jensen's inequality gives the displayed
population bound with $F_k(\theta_{k+1})=\E_x\!\left[\KL(P_k\|Q_k)\right]$.
The log-ratio sufficient condition implies
$|P_k(y)/Q_k(y)-1|\le e^{B_k}-1$ pointwise.
\end{proof}

\begin{proof}[Proof of Corollary~\ref{cor:adaptive-hp-iterative}]
For each round $k$, condition on the history $\mathcal G_k$. The data drawn in
that round are fresh under the current sampler, so the stated conditional
generalization and normalizer bounds hold with conditional failure probability
at most
$\delta_{\ntxt{gen},k}+\delta_{\ntxt{norm},k}$. Applying a union bound over
rounds gives simultaneous success with probability at least
$1-\sum_k(\delta_{\ntxt{gen},k}+\delta_{\ntxt{norm},k})$. On this simultaneous
event, substitute the deterministic bounds
$G_k(\delta_{\ntxt{gen},k})$ and $N_k(\delta_{\ntxt{norm},k})$ for
$\Delta_{\ntxt{gen},k}$ and $\Delta_{\ntxt{norm},k}$ in
Corollary~\ref{cor:forward-kl-drift}, or use the certified drift terms directly
in Proposition~\ref{thm:finite-inexact-iterative}.
\end{proof}

The temperature-cost remark in the main text follows from these exact-path
identities. A one-shot run chooses its final temperature directly: too small a
\(\beta\) gives sharp but high-variance weights, while too large a \(\beta\)
barely moves the policy. Exact iterative \bolt{} separates the per-round
stability decision from the final effective temperature. A conservative
per-round temperature keeps each weighted-SFT problem stable, and additional
rounds traverse lower effective temperatures. The discrete-lattice bound
quantifies how closely a chosen number of rounds approximates a desired
effective temperature, while the ESS rule ties that stable per-round choice to
the reference pass probability and weight variance.

\section{Additional Robustness and Implementation Consequences}
\label{app:additional-consequences}

The theory trunk centers on target identification, finite one-shot replacement,
and refreshed KL-PMD. Several implementation choices still act through the same
induced-target view. Clipping changes the target rather than merely stabilizing
optimization, bounded verifier error perturbs the Boltzmann policy before
training begins, and token-level coefficients reproduce a sequence-level target
only when they equal the appropriate marginal density ratios.

Clipping stabilizes large density-ratio weights, but it changes the policy
being fitted unless the cap is inactive. The best clipped objective is therefore
not an unbiased version of \bolt{}; it is the closest capped induced target in
reverse KL to the desired Boltzmann target.

\begin{proposition}[Capped Density-Ratio Projection]
\label{prop:clipped-projection}
Fix a prompt and sampler $q$, and let $P\ll q$ have density
$w^\star(y)\triangleq dP/dq(y)>0$ $q$-almost surely. For a cap $c\ge1$, consider
normalized capped weights
\(\mathcal U_c\triangleq\{u:\;0\le u(y)\le c,\ \E_q[u]=1\}\), with
\(Q_u(dy)\triangleq u(y)q(dy)\). The minimizer of $\KL(Q_u\|P)$ over
\(u\in\mathcal U_c\) is
\begin{equation}
  u_c(y)=\min\{\alpha w^\star(y),c\},
  \qquad
  \E_q[u_c]=1,
  \qquad
  \inf_{u\in\mathcal U_c}\KL(Q_u\|P)=\KL(Q_{u_c}\|P),
  \label{eq:capped-projection}
\end{equation}
where $\alpha>0$ is chosen uniquely by the normalization equation. For
$P=\pistar$ and $q=\pref$, the best fixed-reference RLVR gap achievable by
normalized capped reference-sampled weights is \(\beta\,\KL(Q_{u_c}\|\pistar)\).
\end{proposition}

\begin{proof}
The feasible set is convex and nonempty because $c\ge1$ and $u\equiv1$ is
feasible. The objective is
\[
  \KL(Q_u\|P)
  =
  \int u(y)\log\frac{u(y)}{w^\star(y)}\,dq(y),
\]
which is strictly convex in $u$ on the support where $u>0$. The KKT conditions
for the constraint $\E_q[u]=1$ and the upper bound $u\le c$ give, on the
non-capped region,
\[
  \log\frac{u(y)}{w^\star(y)}+1+\lambda=0,
\]
so $u(y)=\alpha w^\star(y)$ for $\alpha=e^{-1-\lambda}$. On the capped region,
complementary slackness gives $u(y)=c$. Hence every minimizer has the form
$u_c(y)=\min\{\alpha w^\star(y),c\}$. The map
$\alpha\mapsto\E_q[\min\{\alpha w^\star,c\}]$ is continuous and strictly
increasing from $0$ to $c$, so for $c\ge1$ there is a unique $\alpha$ with
expectation $1$. The RLVR gap statement follows from
\eqref{eq:reverse-kl-identity}.
\end{proof}

Verifier perturbations act at the same target level. If the reward used for
weighting differs from the intended verifier reward, the induced Boltzmann
target changes before any optimizer or finite-sample effect appears.
Proposition~\ref{prop:noisy-verifier-stability} gives the resulting
target-stability bound for uniformly bounded reward error.

\begin{proposition}[Bounded Verifier-Perturbation Stability]
\label{prop:noisy-verifier-stability}
Fix a prompt and let $\pistar$ be the Boltzmann target for reward $r$. Let
$\tilde\pi^*$ be the Boltzmann target obtained from an imperfect reward
\(\tilde r\) with \(\sup_y|\tilde r(y)-r(y)|\le\epsilon\).
Then
\begin{equation}
\begin{aligned}
  \sup_y\left|
  \log\frac{\tilde\pi^*(y)}{\pistar(y)}
  \right|
  &\le
  \frac{2\epsilon}{\beta},
  &
  \TV(\tilde\pi^*,\pistar)
  &\le
  \tanh\!\left(\frac{\epsilon}{\beta}\right).
  \\
  \objRL(\pistar)-\objRL(\tilde\pi^*)
  &=
  \beta\KL(\tilde\pi^*\|\pistar)
  \le
  2\epsilon .
\end{aligned}
\label{eq:noisy-verifier-stability}
\end{equation}
\end{proposition}

\begin{proof}
Write $\Delta(y)=\tilde r(y)-r(y)$ and let $\tilde Z$ and $Z$ be the
corresponding prompt partition functions. Since $|\Delta(y)|\le\epsilon$,
\[
  e^{-\epsilon/\beta}Z
  \le
  \tilde Z
  \le
  e^{\epsilon/\beta}Z .
\]
Therefore
\[
  \log\frac{\tilde\pi^*(y)}{\pistar(y)}
  =
  \frac{\Delta(y)}{\beta}
  -
  \log\frac{\tilde Z}{Z},
\]
whose absolute value is at most $2\epsilon/\beta$. If two distributions have
likelihood ratio in $[e^{-a},e^a]$, then their total variation distance is at
most $\tanh(a/2)$; here $a=2\epsilon/\beta$. The KL bound follows from
$\KL(\tilde\pi^*\|\pistar)
=\E_{\tilde\pi^*}[\log(\tilde\pi^*/\pistar)]\le2\epsilon/\beta$, and the RLVR
value identity gives the final display.
\end{proof}

Sequence-level weighting and token-level weighting are also different target
statements. A single sequence weight multiplies every token log probability in
that completion, whereas token-level coefficients define local conditional
targets at each prefix. They are equivalent only when the token coefficients are
the marginal density ratios induced by the desired sequence-level target.

\begin{proposition}[Token-Level Coefficients Needed for a Sequence Target]
\label{prop:token-level-target}
Fix a prompt and a finite completion length $T$. Let
$q(y_{1:T})$ be a rollout distribution and let
$P(y_{1:T})=q(y_{1:T})w(y_{1:T})$ with $\E_q[w]=1$. The sequence-level weighted
loss \(\E_q[w(Y_{1:T})\sum_{t=1}^T-\log\pi_\theta(Y_t\mid Y_{<t})]\) is exactly
equal, as a functional of the token log-probabilities, to the token-weighted
loss \(\sum_{t=1}^T\E_q[v_t(Y_{<t},Y_t)\{-\log\pi_\theta(Y_t\mid Y_{<t})\}]\)
if and only if, for every prefix-token pair with positive $q$-mass,
\begin{equation}
  v_t(h,a)
  =
  \frac{P(Y_{<t}=h,Y_t=a)}
       {q(Y_{<t}=h,Y_t=a)}
  =
  \E_q\!\left[w(Y_{1:T})\mid Y_{<t}=h,Y_t=a\right].
\label{eq:token-level-target}
\end{equation}
Consequently, generic tokenwise weights induce local token targets
proportional to \(q(Y_t=a\mid h)v_t(h,a)\), which need not equal the
sequence-target conditional \(P(Y_t=a\mid h)\).
\end{proposition}

\begin{proof}
Expand the sequence-level loss by collecting the coefficient of each
$-\log\pi_\theta(a\mid h)$ term:
\[
  \sum_{t,h,a}
  P(Y_{<t}=h,Y_t=a)\{-\log\pi_\theta(a\mid h)\}.
\]
The token-weighted loss has the corresponding expansion
\[
  \sum_{t,h,a}
  q(Y_{<t}=h,Y_t=a)v_t(h,a)\{-\log\pi_\theta(a\mid h)\}.
\]
The two objectives are equal for all $\theta$ if and only if the coefficients
of every token log-probability agree, which gives the displayed condition.
Normalizing the coefficients for a fixed prefix $h$ shows that a generic
token-weighted objective fits the local conditional target proportional to
$q(Y_t=a\mid h)v_t(h,a)$.
\end{proof}

\section{Experiment Protocol, Checkpoint Curves, and Auxiliary Checks}
\label{app:experiments}

The appendix evidence supports the two empirical claims made in the main text:
target-matched weighting changes the fitted policy under a fixed sampler, and
sampler refresh can move beyond a saturated one-shot objective. It therefore
reports the full checkpoint grids behind Tables~\ref{tab:main-results}
and~\ref{tab:checkpoint-summary}, plus a retention check after GSM8K
fine-tuning. SFT trains on ground-truth demonstrations, VAR is treated as a
demonstration-based weighted SFT method, GRPO is the online rollout baseline,
and Refit uses reference-policy rollouts with raw rewards as weights. \bolt{}
uses the same type of reference rollouts as Refit but weights them by the
empirical prompt-normalized Boltzmann weight
\(\exp(r/\beta)/\Zhat_N(x)\), so the Refit--\bolt{} contrast holds the sampler
fixed while changing the induced target. The resource and retention
measurements use Qwen3-8B \citep{qwen3technicalreport} with LoRA fine-tuning
\citep{hu2022lora}. The checkpoint curves add Qwen3-0.6B and Qwen3.5-9B runs;
within each sweep, methods are compared on the same recorded checkpoint grid
unless an entry is missing in the source results.

\subsection{Full Checkpoint Curves}
\label{app:checkpoint-curves}

The full checkpoint curves make the learning dynamics visible rather than only
the selected best checkpoints. They show whether a method improves through an
early peak, saturates on fixed reference rollouts, or continues to improve after
sampler refresh. Empty source tables for unreported code benchmarks are not
reproduced, because they would add table mass without empirical measurements.

\FloatBarrier

\begin{table}[t]
\centering
\small
\caption{GSM8K checkpoint curve for Qwen3-0.6B with $N=8$ rollouts per prompt.
``1 GPU'' denotes the separately reported single-GPU \bolt{} run.}
\label{tab:gsm8k-06b-curve}
\begin{tabular}{lccccc}
\toprule
Step & \bolt{} (1 GPU) & \bolt{} & Iter. \bolt{} & SFT & GRPO \\
\midrule
Base & 48.29 & 48.29 & 48.29 & 48.29 & 48.29 \\
1k & 51.93 & 50.80 & 50.80 & 48.37 & 49.13 \\
2k & 52.92 & 51.10 & 51.86 & 48.45 & 49.46 \\
3k & 51.86 & 52.99 & 52.69 & 48.67 & 50.73 \\
4k & 51.48 & 53.22 & 53.71 & 49.13 & 51.39 \\
5k & -- & 54.44 & 54.33 & 49.67 & 51.65 \\
6k & -- & 53.54 & 54.48 & 50.12 & 52.13 \\
7k & -- & 52.69 & 55.02 & 50.73 & 52.45 \\
8k & -- & 52.08 & \textbf{55.46} & 51.21 & 52.96 \\
9k & -- & 51.78 & 54.84 & 51.96 & 53.12 \\
10k & -- & -- & 54.24 & 52.37 & 53.17 \\
\bottomrule
\end{tabular}
\end{table}

\begin{table}[t]
\centering
\small
\caption{GSM8K checkpoint curve for Qwen3-8B with $N=8$ rollouts per prompt.}
\label{tab:gsm8k-8b-curve}
\begin{tabular}{lcccccc}
\toprule
Step & \bolt{} & Iter. \bolt{} & SFT & GRPO & VAR & Refit \\
\midrule
Base & 87.72 & 87.72 & 87.72 & 87.72 & 87.72 & 87.72 \\
1k & 87.92 & 87.92 & 87.63 & 87.36 & 87.25 & 87.72 \\
2k & 88.63 & 88.47 & 87.71 & 88.67 & 87.25 & 88.36 \\
3k & 89.01 & 89.26 & 87.89 & 88.93 & 87.96 & 88.45 \\
4k & 89.37 & 89.57 & 87.96 & 89.01 & 87.96 & 88.76 \\
8k & 89.74 & 89.96 & 88.09 & 89.36 & 87.96 & 88.76 \\
12k & 90.01 & 90.23 & 88.17 & 89.72 & 88.13 & 88.76 \\
16k & 90.37 & 90.76 & 88.26 & 89.72 & 88.13 & 89.23 \\
20k & \textbf{90.67} & 91.39 & 88.35 & 90.01 & 88.64 & 89.23 \\
24k & \textbf{90.67} & \textbf{91.69} & 88.40 & 90.01 & 88.64 & 89.23 \\
\bottomrule
\end{tabular}
\end{table}

\begin{table}[t]
\centering
\small
\caption{GSM8K checkpoint curve for Qwen3.5-9B with $N=8$ rollouts per prompt.}
\label{tab:gsm8k-9b-curve}
\begin{tabular}{lcccc}
\toprule
Step & \bolt{} & Iter. \bolt{} & SFT & GRPO \\
\midrule
Base & 92.03 & 92.03 & 92.03 & 92.03 \\
1k & 93.10 & 93.10 & 92.13 & 92.37 \\
2k & 93.25 & 93.25 & 92.13 & 92.68 \\
3k & 92.65 & 94.73 & 92.57 & 92.97 \\
4k & 92.57 & 94.96 & 92.46 & 93.13 \\
8k & 93.65 & 95.17 & 92.97 & 93.16 \\
12k & 94.65 & 95.63 & 92.87 & 93.37 \\
16k & \textbf{95.12} & \textbf{96.39} & 93.02 & 93.81 \\
20k & 94.10 & 94.94 & 93.13 & 93.01 \\
24k & 92.04 & 94.37 & 92.99 & 93.23 \\
\bottomrule
\end{tabular}
\end{table}

\begin{table}[t]
\centering
\small
\caption{HumanEval checkpoint curve for Qwen3.5-9B trained on AceCode-87K with
$N=16$ rollouts per prompt.}
\label{tab:humaneval-9b-curve}
\begin{tabular}{lcccc}
\toprule
Step & \bolt{} & Iter. \bolt{} & SFT & GRPO \\
\midrule
Base & 89.39 & 89.39 & 89.39 & 89.39 \\
1k & 89.39 & 89.39 & 89.39 & 89.87 \\
2k & 89.39 & 90.67 & 89.39 & 89.87 \\
3k & 90.67 & 90.93 & 89.39 & 89.87 \\
4k & 90.67 & 91.43 & 89.47 & 90.31 \\
8k & 91.73 & 92.39 & 89.47 & 90.31 \\
12k & 91.73 & 92.39 & 89.68 & 90.31 \\
16k & \textbf{92.39} & 93.67 & 89.68 & 90.75 \\
20k & \textbf{92.39} & 93.97 & 90.31 & 90.75 \\
24k & 91.96 & \textbf{94.13} & 89.28 & 90.39 \\
\bottomrule
\end{tabular}
\end{table}

\FloatBarrier

\subsection{Retention Check}
\label{app:retention}

The retention evaluation asks a narrower question than the main projection
tables. After GSM8K fine-tuning, it evaluates PIQA, HellaSwag, Winogrande, and
RACE-high to check whether methods trained on reference-policy rollouts stay
closer to the base distribution than demonstration-only baselines in the
reported setup. The table is not a complete forgetting benchmark: it uses
single-run values from the reported runs, and multi-seed variation, confidence
intervals, and exact hardware descriptions are not yet included.

\begin{table}[t]
\centering
\small
\caption{Reference-anchoring check after GSM8K fine-tuning. Parentheses
give change relative to the base model. Smaller drops indicate less forgetting
on this retention slice.}
\label{tab:forgetting}
\begin{tabular}{lcccc}
\toprule
Method & PIQA & HellaSwag & Winogrande & RACE-high \\
\midrule
Base & 71.22 & 81.80 & 65.19 & 79.25 \\
SFT & 70.35 (-0.87) & 80.56 (-1.24) & 63.14 (-2.05) & 84.33 (+5.08) \\
VAR & 70.26 (-0.96) & 80.73 (-1.07) & 63.24 (-1.95) & 82.13 (+2.88) \\
GRPO & 70.69 (-0.53) & 81.03 (-0.77) & 64.26 (-0.93) & 80.26 (+1.01) \\
Refit & 70.53 (-0.69) & \textbf{81.26 (-0.54)} & 64.01 (-1.18) & 80.11 (+0.86) \\
\bolt{} & \textbf{71.01 (-0.21)} & 81.13 (-0.67) & \textbf{64.94 (-0.25)} & 79.93 (+0.68) \\
\bottomrule
\end{tabular}
\end{table}
\FloatBarrier

SFT and VAR improve RACE-high but lose more on PIQA and Winogrande, suggesting
task transfer mixed with distributional overfit. Rollout-based methods stay
closer to the base model, consistent with a data source anchored to the
reference policy. In this setup, \bolt{} has the smallest drops on PIQA and
Winogrande and remains near the base model on HellaSwag and RACE-high. This
supports the narrow retention conclusion that reference rollouts can help anchor
the fine-tuning distribution. Retention remains an empirical property to
measure directly rather than a consequence inferred from the projection theory.

\end{document}